\documentclass[]{article}

\usepackage{hyperref}       % hyperlinks
\usepackage{url}            % simple URL typesetting
\usepackage{booktabs}       % professional-quality tables
\usepackage{amsfonts}       % blackboard math symbols

\usepackage{natbib}

\usepackage{graphicx}
\usepackage{subcaption}

\usepackage{tcolorbox}
\definecolor{darkred}{RGB}{150,0,0}
\definecolor{darkgreen}{RGB}{0,150,0}
\definecolor{darkblue}{RGB}{0,0,150}
\hypersetup{colorlinks=true, linkcolor=red, citecolor=darkblue, urlcolor=darkblue}

\usepackage{tikz}
\usepackage{amssymb}
\usepackage{amsmath}
\usepackage{amsmath,amsfonts,amssymb,amsthm}
\usepackage{mathtools}
\usepackage{commath}
\usepackage{flexisym}
\usepackage[utf8]{inputenc}
\usepackage{csquotes}
\usepackage[english]{babel}

\usepackage{color}
\usepackage{bbm}
\usepackage{multirow}
\usepackage{amsmath}
\DeclareMathOperator*{\argmax}{arg\,max}

\newtheorem{lemma}{Lemma}
\newtheorem{corollary}{Corollary}
\newtheorem{theorem}{Theorem}
\newtheorem{myassum}{Assumption}
\theoremstyle{definition}
\newtheorem{remark}{Remark}
\theoremstyle{definition}
\newtheorem{definition}{Definition}

\usepackage[shortlabels]{enumitem}
\usepackage{lipsum}
\usepackage{xcolor}
\usepackage[linesnumbered,ruled,vlined]{algorithm2e}
\usepackage{enumitem}

\SetCommentSty{mycommfont}

\SetKwInput{KwInput}{Input}                % Set the Input
\SetKwInput{KwOutput}{Output} 

\usepackage[margin=1in]{geometry}

\title{Distributed Contextual Linear Bandits\\with Minimax Optimal Communication Cost}

\newcommand{\lamin}{\la_{\rm \min}}
\newcommand{\lamax}{\la_{\rm \max}}
\newcommand{\la}{\lambda}

\newcommand{\nn}{\nonumber}

\usepackage{xspace}
\usepackage[textsize=tiny]{todonotes}

\newcommand{\Unif}{{\rm Unif}}
\newcommand{\DKL}{{\rm D_{KL}}}

\newcommand{\DisBE}{DisBE-LUCB\xspace}
\newcommand{\DecBE}{DecBE-LUCB\xspace}

\newcommand{\bal}{\begin{align}}
\newcommand{\eal}{\end{align}}

\newcommand{\M}{\mathbf{M}}
\newcommand{\W}{\mathbf{W}}
\newcommand{\Ub}{\mathbf{U}}

\newcommand{\D}{\mathbf{D}}
\newcommand{\C}{\mathbf{C}}
\newcommand{\Pb}{\mathbf{P}}

\newcommand{\X}{\mathbf{X}}
\newcommand{\A}{\mathbf{A}}
\newcommand{\B}{\mathbf{B}}
\newcommand{\Yb}{\mathbf{Y}}
\newcommand{\Vb}{\mathbf{V}}
\newcommand{\x}{\mathbf{x}}

\newcommand{\ub}{\mathbf{u}}

\newcommand{\bb}{\mathbf{b}}
\newcommand{\e}{\mathbf{e}}
\newcommand{\y}{\mathbf{y}}

\newcommand{\sind}{\mathbf{1}}
\newcommand{\cT}{\mathcal T}
\newcommand{\z}{\mathbf{z}}

\newcommand{\Tc}{{\mathcal{T}}}
\newcommand{\Sc}{{\mathcal{S}}}

\newcommand{\Dc}{\mathcal{D}}
\newcommand{\Xc}{\mathcal{X}}

\newcommand{\Kc}{\mathcal{K}}
\newcommand{\Nc}{\mathcal{N}}

\newcommand{\Cc}{\mathcal{C}}
\newcommand{\Mc}{\mathcal{M}}

\newcommand{\Ec}{\mathcal{E}}

\newcommand{\Oc}{\mathcal{O}}

\newcommand{\beq}{\begin{equation}}
\newcommand{\eeq}{\end{equation}}
\newcommand{\bea}{\begin{align}}
\newcommand{\eea}{\end{align}}

\newcommand{\E}{\mathbb{E}}

\newcommand{\Otilde}{\tilde\Oc}

\newcommand{\mub}{\boldsymbol \mu}

\newcommand{\nub}{\boldsymbol \nu}
\newcommand{\Sigmab}{\boldsymbol \Sigma}

\newcommand{\thetab}{\boldsymbol\theta}

\newcommand{\Comm}{{\rm Comm}}
\newcommand{\Exp}{{\rm ExpPol}}

\newcommand{\MixedSoftMax}{{\rm MixedSoftMax}}

\newcommand{\piG}{\pi^{\rm G}}

\newcommand{\Vlambda}{{{\mathbb{V}}}}

\newcommand{\ip}[1]{\langle #1 \rangle}

\usepackage{authblk}
\author[1]{Sanae Amani}
\author[2]{Tor Lattimore}
\author[3]{András György}
\author[4]{Lin F. Yang}
\affil[1,4]{University of California, Los Angeles}
\affil[2,3]{DeepMind, London}
{
    \makeatletter
    \renewcommand\AB@affilsepx{, \protect\Affilfont}
    \makeatother
    \affil[1]{samani@ucla.edu}
\affil[2]{lattimore@google.com}
\affil[3]{agyorgy@deepmind.com}
\affil[4]{linyang@ee.ucla.edu}
}

\begin{document}

\sloppy
\date{}
\maketitle

\begin{abstract}
We study distributed contextual linear bandits with stochastic contexts, where $N$ agents act cooperatively to  solve a linear bandit-optimization problem with $d$-dimensional features over the course of $T$ rounds. For this problem, we derive the first ever information-theoretic lower bound $\Omega(dN)$ on the communication cost of any algorithm that performs optimally in a regret minimization setup. We then propose a distributed batch elimination version of the LinUCB algorithm, \DisBE, where the agents share information among each other through a central server. We prove that the communication cost of \DisBE matches our lower bound up to logarithmic factors. In particular, for scenarios with known context distribution, the communication cost of \DisBE is only $\Otilde(dN)$ and its regret is $\Otilde(\sqrt{dNT})$, which is of the same order as that incurred by an optimal single-agent algorithm for $NT$ rounds. We also provide similar bounds for practical settings where the context distribution can only be estimated. Therefore, our proposed algorithm is nearly minimax optimal in terms of \emph{both regret and communication cost}. Finally, we propose \DecBE, a fully decentralized version of \DisBE, which operates without a central server, where agents share information with their \emph{immediate neighbors} through a carefully designed consensus procedure. 
\end{abstract}

\section{INTRODUCTION}
In the contextual bandit problem, a learning agent repeatedly makes decisions based on contextual information, with the goal of learning a policy that maximizes their total reward over time. This model captures simple reinforcement learning tasks in which the agent must learn to make high-quality decisions in an uncertain environment, but does not need to engage in long-term planning. Contextual bandit algorithms are deployed in online personalization systems such as medical trials and product recommendation in e-commerce
% and adaptive experiment design in crowdsourcing
\citep{agarwal2016making,tewari2017ads}. For example, by modelling personalized recommendation of articles as a contextual bandit problem, a learning algorithm sequentially selects articles to be recommended to users based on contextual information about the users and articles, while continuously updating its article-selection strategy based on user-click feedback to maximize total user clicks \citep{li2010contextual}.

Distributed cooperative learning is a paradigm where multiple agents collaboratively learn a shared prediction model. More recently, researchers have explored the potential of contextual bandit algorithms in distributed systems, such as in robotics, wireless networks, the power grid and medical trials \citep{li2013medicine,avner2019multi,berkenkamp2016bayesian,sui2018stagewise}. For example, in sensor/wireless networks \citep{avner2019multi} and channel selection in radio networks \citep{liu2010decentralized,liu2010distributed,liu2010distributed1}, a collaborative behavior is required for decision-makers/agents to select better actions as individuals. 
 
While a distributed nature is inherent in certain systems, distributed solutions 
might also be preferred in broader settings, as they can lead to speed-ups of the learning process. This calls for extensions of the traditional single-agent bandit setting to networked systems. In addition to speeding up the learning process, another desirable goal of each distributed learning algorithm is \emph{communication efficiency}. In particular, keeping the communication as rare as possible in collaborative learning is of importance. The notion of communication efficiency in distributed learning paradigms is directly related to the issue of efficient environment queries made in single-agent settings. In many practical single-agent scenarios, where the agent sequentially
makes active queries about the environment, it is desirable to limit these queries
to a small number of rounds of interaction, which helps to increase the parallelism of the learning process and reduce the management cost. In recent years, to address such scenarios, a surge of research activity in the area of batch online learning has shown that in many popular online learning tasks, a very small number of batches
may achieve minimax optimal learning performance, and therefore it is possible to enjoy the benefits
of both adaptivity and parallelism \citep{ruan2021linear,han2020sequential,gao2019batched}. In light of the connection between communication cost in distributed settings and the number of environment queries in single-agent settings, a careful use of batch learning methods in multi-agent learning scenarios may positively affect the communication efficiency by limiting the number of necessary communication rounds. In this paper, we first prove an information-theoretic lower bound on the communication cost of distributed contextual linear bandits, and then leverage such batch learning methods to design an algorithm with a small communication cost that matches this lower bound while guaranteeing optimal regret.

\textbf{Notation.} 
Throughout this paper, we use lower-case letters for scalars, lower-case bold letters for vectors, and upper-case bold letters for matrices. The Euclidean norm of $\x$ is denoted by $\norm{\x}_2$. We denote the transpose of any column vector $\x$ by $\x^\top$. For any vectors $\x$ and $\y$, we use $\langle \x,\y\rangle$ to denote their inner product. Let $\A$ be a positive semi-definite $d\times d$ matrix and $\boldsymbol \nu \in\mathbb R^d$. The weighted 2-norm of $\boldsymbol \nu$ with respect to $\A$ is defined by $\norm{\boldsymbol \nu}_\A = \sqrt{\boldsymbol \nu^\top \A \boldsymbol \nu}$. For a positive integer $n$, $[n]$ denotes the set $\{1,2,\ldots,n\}$, while for positive integers $m \leq n$, $[m:n]$ denotes the set $\{m,m+1,\ldots,n\}$. For square matrices $\A$ and $\B$, we use $\A\preceq \B$ to denote $\B-\A$ is positive semi-definite. We denote the minimum and maximum eigenvalues of $\A$ by $\lamin(\A)$ and $\lamax(\A)$. We use $\mathbf{e}_i$ to denote the $i$-th standard basis vector. $I(X;Y)$ denotes the mutual information between two random variables $X$ and $Y$.
% For a distribution $\Dc$ over the $d$-dimensional space, we denote its support by $\Supp(\Dc)=\{\x\in\mathbb{R}^d:\mathbb{P}_{\Dc}(\x)>0\}$.
Finally, we use standard $\Otilde$ notation for big-O notation that ignores logarithmic factors.

%%%%%%%%%%%%%%%%%%%%%%%%%%%%%%%%%%%%%%%%%%%%%%%%%%%%%%%%%%%%%%%%%%%%%%%%%%%%%%%%%%%%%%%%%%%%%%%%%%%%%%%%%%%%%%%%%%%%%%%%%%%%%%%%%%%%%%%%%%%%%%%%%%%%%%%%%%%%%%%%%%%%%%%%%%%%%%%%%%%%%%%%%%%%%%%%%%%%%%%%%%%%%%%%%%%%%%%%%%%%%%%%%%%%%%%%%%%%%%%%%%%%%%%%%%%%%%%

\subsection{Problem formulation} \label{sec:formulate}

We consider a network of $N$ agents acting cooperatively to efficiently solve a $K$-armed stochastic linear bandit problem. Let $T$ be the total number of rounds. At each round $t\in[T]$, each agent $i$ is given a decision set $\Xc_t^i=\{\x^i_{t,a}:a\in[K]\}\subset\mathbb{R}^d$, drawn independently from a distribution $\Dc_t^i$. We assume that $\Dc_t^i = \Dc$ for all $(i,t)\in[N]\times[T]$. Here, $\x_{t,a}^i$ is a mapping from action $a$ and the contextual information agent $i$ receives at round $t$ to the $d$-dimensional space. We call $\x_{t,a}^i$ the feature vector associated with action $a$ and agent $i$ at round $t$. Agent $i$ selects action $a_{i,t}\in[K]$, and observes the reward $y^{i}_{t}=\langle \thetab_\ast,\x^i_{t,a_{i,t}}\rangle+\eta_t^i$, where $\thetab_\ast\in\mathbb{R}^d$ is an unknown vector and $\eta_t^i$ is an independent zero-mean additive noise. The agents are also allowed to communicate with each other. Both the action selection and the communicated information of each agent may only depend on previously played actions, observed rewards, decision sets, and communication received from other agents.

Throughout the paper, we rely on the following assumption.
\begin{myassum}\label{assum:boundedness} Without loss of generality, $\norm{\boldsymbol\theta_\ast}_2\leq 1$, $\norm{\x^i_{t,a}}_2\leq 1$, $\abs{y_t^i}\leq 1$  for all $(a,i,t)\in[K]\times[N]\times[T]$. Also, the distribution $\Dc$ is known to the agents.
% and the context vectors distribution $\Dc$ is known to the agents.
\end{myassum}
The boundedness assumption is standard in the linear bandit literature \citep{chu2011contextual,dani2008stochastic,huang2021federated}. Moreover, our results can be readily extended to the settings where the assumption on the boundedness of $y_t^i$ is relaxed by assuming the noise variables $\eta_t^i$ are conditionally $\sigma$-subGaussiam for a constant $\sigma\geq 0$. As such, a high probability bound on $\eta_t^i$ and consequently $y_t^i$ can be established, which is desired in our analysis for establishing confidence intervals in Appendix \ref{sec:proofofconfidencesets}.

Our assumption on the knowledge of $\Dc$ is fairly well-motivated. A standard argument is based on having loads of unsupervised data in real-world scenarios. For example, Google, Amazon, Netflix, etc, have collected massive amounts of data about users, products, and queries, sufficiently describing the joint distributions. Given this, even if the features change (for a given user or product, etc.), their distributions can be computed/sampled from as the features are computed via a deterministic feature map. We further relax this assumption in Remark \ref{remark:relaxD} in Section \ref{sec:main result}.

\paragraph{Goal.}
The performance of the network is measured via the cumulative regret of all agents in $T$ rounds, defined as
\begin{align}\label{eq:cumulativeregret}
    R_T := \mathbb{E}\left[\sum_{t=1}^T\sum_{i=1}^N \langle\thetab_\ast,\x^i_{\ast,t}\rangle-\langle\thetab_\ast,\x^i_{t}\rangle\right],
\end{align}
where the expectation is taken over the random variables $\Xc_t^i, (i,t) \in [N]\times[T]$ 
%random set $\bigotimes_{i,t=1}^{N,T}\Xc_t^i$ 
with joint distribution $\bigotimes_{i,t=1}^{N,T}\Dc_t^i$, $\x_t^i$ and $\x^i_{\ast,t}\in\argmax_{\x\in \Xc_t^i}\langle\thetab_\ast,\x\rangle$ are the feature vectors associated with the action chosen by agent $i$ at round $t$ and the best possible action, respectively.

For simplicity, in our algorithms the communication cost is measured as the number of communicated real numbers. In Section \ref{sec:lowerbound}, we also discuss variants of our methods where the communication cost is measured as the number of communicated bits.

The goal is to design a distributed collaborative algorithm that minimizes the cumulative regret, while maintaining an efficient coordination protocol with a small communication cost. Specifically, we wish to achieve a regret close to $\tilde\Oc(\sqrt{dNT})$ that is incurred by an optimal \emph{single-agent algorithm for $NT$ rounds} (the total number of arm pulls) while the communication cost is $\tilde\Oc(dN)$ with only a mild (logarithmic) dependence on $T$.

\paragraph{A motivating example.}

% In clinical trials, the candidate actions correspond to the $K$ involved treatments.
% At round $t$, an individual patient arrives with the context vectors $\{\x_{t,a}:a\in[K]\}$ characterizing its
% response to the candidate treatments, and the recovery probability given treatment $a$ is modeled
% by the linear function $\langle \thetab_\ast,\x_{t,a}\rangle$, which corresponds to the expected reward. 

In news article recommendation, the candidate actions correspond to $K$ news articles.
At round $t$, an individual user visits an online news platform that has $N$ servers employing the same recommender systems to recommend news articles from an article pool. The contextual information of the user, the articles and the servers at round $t$ is modeled by $\Xc^i_t=\{\x^i_{t,a}:a\in[K]\}$, characterizing user's
reaction to each recommended article $a$ (e.g., click/not click) by server $i$, and the probability of clicking on $a$ is modeled
by $\langle \thetab_\ast,\x^i_{t,a}\rangle$, which corresponds to the expected reward. On the distributed side, these $N$ servers collaborate with each other by sharing information about the feedback they receive from the users after recommending articles in an attempt to speed up learning the users' preferences. In this example, the individual users and articles can often
be viewed as independent samples from the population which is characterized by distribution $\Dc$.

%%%%%%%%%%%%%%%%%%%%%%%%%%%%%%%%%%%%%%%%%%%%%%%%%%%%%%%%%%%%%%%%%%%%%%%%%%%%%%%%%%%%%%%%%%%%%%%%%%%%%%%%%%%%%%%%%%%%%%%%%%%%%%%%%%%%%%%%%%%%%%%%%%%%%%%%%%%%%%%%%%%%%%%%%%%%%%%%%%%%%%%%%%%%%%%%%%%%%%%%%%%%%%%%%%%%%%%%%%%%%%%%%%%%%%%%%%%%%%%%%%%%%%%%%%%%%%%

\subsection{Contributions}\label{sec:contributions}

We establish a lower bound on the communication cost of distributed contextual linear bandits. We propose algorithms with optimal regret and communication cost matching our lower bound (up to logarithmic factors) and growing linearly with $d$ and $N$ while those of previous best-known algorithms scale super linearly either in $d$ or $N$. Below, we elaborate more on our contributions:

\begin{table*}[t!]
% \tiny
\scriptsize
% \footnotesize
%  \small
\begin{tabular}{|p{2.74cm}|p{2cm}|p{5cm}|p{3.1cm}|p{1.4cm}| } 
  \hline
  Setting & Algorithm & Regret & Communication cost & Communication cost lower bound\\ 
  \hline
  Contexts are fixed over time horizon and agents &DELB with server \cite{wang2019distributed} &$\Oc\left(d\sqrt{NT\log T}\right)$ & $\Oc\left((dN+d\log\log d)\log T\right)$ &  \\ 
  \hline 
  \multirow{2}{3cm} {Contexts adversarially\\ vary over time horizon\\ and agents}
  & DisLinUCB with server \cite{wang2019distributed} &$\Oc\left(d\sqrt{NT}\log^2 T\right)$& $\Oc\left(d^3N^{1.5}\right)$ &  \\ 
  & FedUCB with server \cite{dubey2020differentially}&  $\Oc\left(d\sqrt{NT}\log^2 T\right)$&$\Oc\left(d^3N^{1.5}\right)$&\\
  \hline
  Contexts adversarially vary over agents & Fed-PE with server \cite{huang2021federated} &  $\Oc\left(\sqrt{dNT\log (KNT)}\right)$ & $\Oc\left((d^2+dK)N\log T\right)$ & \\
  \hline
  \multirow{2}{3cm} {Contexts stochastically\\ vary over time horizon\\ and agents ({\bf this work})
  } 
&   \DisBE with server & $\Oc\left(\sqrt{dNT\log d \log^2\left(KNT\right)}\right)$ & $\Oc\left(dN\log\log(NT)\right)$ & $\Omega(dN)$\\
&   \DecBE without server& $\Oc\left(NS+\sqrt{dN(T+S)\log d \log^2\left(KNT\right)}\right)$ & $\Oc\left(S\delta_{\rm max}dN\log\log(NT)\right)$ & \\
  \hline
\end{tabular}
\caption{$N$: number of agents; $K$: number of arms; $T$: time horizon; $d$: dimension of the feature vectors; $S= \frac{\log(dN)}{\sqrt{1/\abs{\la_2}}}$;  $\abs{\la_2}$: the second largest eigenvalue of communication matrix in absolute value; $\delta_{{\rm max}}$ is the maximum degree of the graph representing agents' network. 
The lower bound for the communication cost is interpreted as follows: For any algorithm with expected communication cost less than $\frac{dN}{64}$, there exists a contextual linear bandit instance with stochastic contexts, for which the algorithm's regret is $\Omega(N\sqrt{dT})$. See Theorem~\ref{thm:lowerbound} and its proof in Section \ref{sec:proofoflowerbound} for more details.}
\label{table:comp}
\end{table*}

\paragraph{Minimax lower bound for the communication cost.} As our main technical contribution, in Section~\ref{sec:lowerbound}, we prove the first information-theoretic lower bound on the communication cost (measured in bits) of any algorithm achieving an optimal regret rate for the distributed contextual linear bandit problem with stochastic contexts. In particular, we prove that for any distributed algorithm with expected communication cost less than $\frac{dN}{64}$, there exists a contextual linear bandit problem instance with stochastic contexts for which the algorithm's regret is $\Omega(N\sqrt{dT})$. 

\paragraph{\DisBE.} 
We propose a distributed batch elimination contextual linear bandit algorithm (\DisBE): the time steps are grouped into $M$ pre-defined batches and at each time step, each agent first constructs confidence intervals for each action's reward, and the actions whose confidence intervals completely fall below those of other actions are eliminated. Throughout each batch, each agent uses the same policy to select actions from the surviving action sets. At the end of each batch, the agents share information through a central server and update the policy they use in the next batch. 
We prove that while the communication cost of \DisBE is only $\Otilde(dN)$, it achieves a regret $\Otilde(\sqrt{dNT})$, which is of the same order as that incurred by a near optimal \emph{single-agent algorithm for $NT$ rounds} . This shows that \DisBE is nearly minimax optimal in terms of \emph{both regret and communication cost}. We highlight that while \DisBE is inspired by the single-agent batch elimination style algorithms \citep{ruan2021linear} in an attempt to save on communication as much as possible, a direct use of confidence intervals used in such algorithms would fail to guarantee optimal communication cost $\Otilde(dN)$ and require more communication by a factor of $\Oc(d)$. We address this issue by introducing new confidence intervals in Lemma \ref{lemm:confidencesets}.
Details are given in Section~\ref{sec:FPE}.

\paragraph{\DecBE.} Finally, we propose a fully decentralized variant of \DisBE without a central server, where the agents can only communicate with their \emph{immediate neighbors} given by a communication graph. Our algorithm, called decentralized batch elimination linear UCB (\DecBE), runs a carefully designed consensus procedure to spread information throughout the network. For this algorithm, we prove a regret bound that captures both the degree of selected actions' optimality and the inevitable delay in information-sharing due to the network structure while the communication cost still grows linearly with $d$ and $N$. See Section \ref{sec:decentralized}.

 We complement our theoretical results with numerical simulations under various settings in Section~\ref{sec:experiments}.

%%%%%%%%%%%%%%%%%%%%%%%%%%%%%%%%%%%%%%%%%%%%%%%%%%%%%%%%%%%%%%%%%%%%%%%%%%%%%%%%%%%%%%%%%%%%%%%%%%%%%%%%%%%%%%%%%%%%%%%%%%%%%%%%%%%%%%%%%%%%%%%%%%%%%%%%%%%%%%%%%%%%%%%%%%%%%%%%%%%%%%%%%%%%%%%%%%%%%%%%%%%%%%%%%%%%%%%%%%%%%%%%%%%%%%%%%%%%%%%%%%%%%%%%%%%%%%%

\section{RELATED WORK}\label{sec:relatedwork}

\paragraph{Distributed MAB.} Multi-armed bandit (MAB) in multi-agent distributed settings has received attention from several academic communities. In the context of the classical $K$-armed MAB, \citet{martinez2019decentralized,landgren2016distributed,landgr,landgren2018social} proposed decentralized algorithms for a network of $N$ agents that can share information only with their immediate neighbors, while \citet{szorenyi2013gossip} studied the MAB problem on peer-to-peer networks.

\paragraph{Distributed contextual linear bandits.} The most closely related  works on distributed linear bandits are those of \citet{wang2019distributed,dubey2020differentially,huang2021federated,korda2016distributed,hanna2022learning}. In particular, \citet{wang2019distributed} investigate communication-efficient distributed linear bandits, where the agents can communicate with a server by sending and receiving packets. They propose two algorithms, namely, DELB and DisLinUCB, for fixed and time-varying action sets, respectively. The works of \citet{dubey2020differentially, huang2021federated} consider the federated  linear contextual bandit model and the former focuses on federated differential privacy. In the latter, the contexts denote the specifics of the agents and are different but fixed during the entire time horizon for each agent. In the former, however, the contexts contain the information about both the environment and the agents, in the sense that contexts associated with different agents are different and vary during the time horizon. To put these in the context of an example, consider a recommender system. Both \citet{dubey2020differentially} and \citet{huang2021federated} consider a multi-agent
model, where each agent is associated with a different user profile. \citet{huang2021federated} fix a user profile for an agent, while \citet{dubey2020differentially} consider a time-varying user profile. Therefore, \citet{huang2021federated} capture the variation of contexts over agents, whereas it is captured over both agents and time horizon in \citet{dubey2020differentially}. A regret and communication cost comparison between \DisBE, \DecBE and other baseline algorithms is given in Table~\ref{table:comp}.

\paragraph{Batch elimination in distributed bandits.} An important line of work related to communication efficiency in distributed bandits studies practical single-agent scenarios using batch elimination methods, in which a very small number of batches achieve minimax optimal learning performance \citep{ruan2021linear,han2020sequential,gao2019batched}. Our proposed algorithms are inspired by the single-agent BatchLinUCB-DG proposed in \citet{ruan2021linear} in an attempt to save on communication as much as possible. That said, a direct use of confidence intervals in \citet{ruan2021linear} would fail to guarantee optimal communication cost $\Otilde(dN)$ and require more communication by a factor of $\Oc(d)$. We address this issue by introducing new confidence intervals, used in our algorithms, in Lemma \ref{lemm:confidencesets}.

\paragraph{Minimax lower bound on communication cost.} We are unaware of any lower bound on the communication cost for contextual linear bandits in the distributed/federated learning setting. To the best of our knowledge, our work is the first to establish such a minimax lower bound and to propose algorithms with optimal regret and communication cost matching this lower bound up to logarithmic factors.
%%%%%%%%%%%%%%%%%%%%%%%%%%%%%%%%%%%%%%%%%%%%%%%%%%%%%%%%%%%%%%%%%%%%%%%%%%%%%%%%%%%%%%%%%%%%%%%%%%%%%%%%%%%%%%%%%%%%%%%%%%%%%%%%%%%%%%%%%%%%%%%%%%%%%%%%%%%%%%%%%%%%%%%%%%%%%%%%%%%%%%%%%%%%%%%%%%%%%%%%%%%%%%%%%%%%%%%%%%%%%%%%%%%%%%%%%%%%%%%%%%%%%%%%%%%%%%%

\section{LOWER BOUND ON COMMUNICATION COST}\label{sec:lowerbound}
In this section, we derive an information-theoretic lower bound on the communication cost of the distributed contextual linear bandits with stochastic contexts. In particular, we prove that for any distributed contextual linear bandit algorithm with stochastic contexts that achieves the optimal regret rate $\Otilde(\sqrt{dNT})$, the expected amount of communication must be at least $\Omega(dN)$. This is formally stated in the following theorem.

\begin{theorem}\label{thm:lowerbound}
Let $T\geq 4d\log(8)$. For any algorithm with expected communication cost (measured in bits) less than $\frac{dN}{64}$, there exists a contextual linear bandit instance with stochastic contexts, for which the algorithm's regret is $\Omega(N\sqrt{dT})$.
\end{theorem}
% In the next section, we give a sketch of the proof.

%%%%%%%%%%%%%%%%%%%%%%%%%%%%%%%%%%%%%%%%%%%%%%%%%%%%%%%%%%%%%%%%%%%%%%%%%%%%%%%%%%%%%%%%%%%%%%%%%%%%%%%%%%%%%%%%%%%%%%%%%%%%%%%%%%%%%%%%%%%%%%%%%%%%%%%%%%%%%%%%%%%%%%%%%%%%%%%%%%%%%%%%%%%%%%%%%%%%%%%%%%%%%%%%%%%%%%%%%%%%%%%%%%%%%%%%%%%%%%%%%%%%%%%%%%%%%%%

\subsection{Proof of Theorem \ref{thm:lowerbound}}\label{sec:proofoflowerbound}
We start with a lower bound for a Bayesian two-armed bandit problem where the learner is given side information that
contains a small amount of information about the optimal action.

\begin{lemma}\label{lemm:singleagentlowerbound}
Let $\mub_1 = (\Delta, 0)$ and $\mub_2 = (-\Delta, 0)$ and consider the Bayesian two-armed Gaussian bandit with mean $\mub$ uniformly
sampled from $\{\mub_1, \mub_2\}$ and $a_\ast = \argmax_{a \in \{1,2\}} \mub_a$, which is a random variable.
Suppose additionally that the learner has access to a random element $M$ with $I(M; a_\star) \leq 1/16$. Then, for any policy $\pi$,
\begin{align*}
    BR_T(\pi)\geq \Delta T\left(\frac{1}{2}-\sqrt{\frac{1}{2}\left(\frac{1}{16}+4T\Delta^2\right)}\right)\,,
\end{align*}
where $BR_T(\pi) = \mathbb{E}_{\mub\sim\Unif\{\mub_1,\mub_2\}}[R_T(\pi,\mub)]$ and $R_T(\pi, \mub)$ is the regret suffered by policy
$\pi$ in the Gaussian two-armed bandit with means $\mub$.
\end{lemma}

\begin{remark}\label{rem:sa}
We assume in Lemma~\ref{lemm:singleagentlowerbound} that the learner has access to the message $M$ from the beginning.
The same bound continues to hold in the strictly harder problem where the learner has sequential access 
to a sequence of messages $M_1,\ldots,M_T$ with $I(\{M_t\}_{t=1}^T ; a_\ast) \leq 1/16$.
\end{remark}

The proof is presented in Appendix \ref{sec:proofofmutual}. This lemma emphasizes the role of extra information a single agent might receive throughout the learning process on its performance, and therefore, it is key in proving Theorem \ref{thm:lowerbound}.  Specifically, since Lemma \ref{lemm:singleagentlowerbound} makes no assumption on how the agent receives the extra information about the learning environment, we can prove Theorem \ref{thm:lowerbound} by employing this lemma and a reduction from single-agent bandit to multi-agent
bandit as explained in what follows.

\paragraph{The construction.} We consider a bandit instance where $K=2$ and the decision sets are drawn uniformly from  $\left\{(\e_1,\e_2),(\e_3,\e_4),\ldots,(\e_{d-1},\e_d)\right\}$. Let $\Theta = \{\thetab\in\mathbb{R}^d:(\thetab_{2j-1},\thetab_{2j})\in\{(\Delta,0), (-\Delta,0)\},~\forall j\in[\frac{d}{2}]\}$. We call $(\thetab_{2j-1},\thetab_{2j})$ by $j$-th block of reward vector.

\paragraph{Bayesian regret.}
As in Lemma~\ref{lemm:singleagentlowerbound} we prove the minimax-style lower bound using the Bayesian regret. 
Let $\thetab$ be sampled uniformly from $\Theta$ and $\pi$ be a fixed policy. The Bayesian regret is
\begin{align*}
BR_T = \E\left[\sum_{t=1}^T \sum_{i=1}^N \ip{\thetab, \x^i_{\ast,t}} - \ip{\thetab, \x^i_t}\right]\,,
\end{align*}
where the expectation integrates over the randomness in both $\thetab$ and the corresponding history induced by the interaction
between $\pi$ and the environment determined by $\thetab$.
By Yao's minimax principle, there exists a $\thetab \in \Theta$ such that the expected regret is at least $BR_T$, so it suffices to
lower bound the Bayesian regret.
For the remainder of the proof $\mathbb E[\cdot]$ and $\mathbb P(\cdot)$ correspond to the expectation and probability 
measure on $\thetab$ and the history.
For technical reasons, we assume that these probability spaces are defined to include an infinite interaction between the learner
and environment. Of course, this is only used in the analysis.

\paragraph{Reduction from single-agent bandit to multi-agent.}
Let $M_{ij}$ be the mutual information between messages agent $i$ receives and $j$-th block of the reward vector $(\thetab_{2j-1},\thetab_{2j})$. 
By assumption,
\begin{align}
    \sum_{i=1}^N\sum_{j=1}^{\frac{d}{2}}M_{ij}&\leq \sum_{i=1}^N\mathbb{E}[\text{Total number of bits agent $i$ receives}]\leq \frac{dN}{64}.\label{eq:ccdN128mainbody}
\end{align}
% For each agent $i\in[N]$, let $\Ec_i$ be the set of $\frac{d}{4}$ blocks of reward vector with smallest $M_{ij}$. Furthermore, define $\Sc$ as the set of $\frac{N}{2}$ agents with smallest $\sum_{j\in\Ec_i}M_{ij}$. From \eqref{eq:ccdN128mainbody} and the definition of $\Sc$ and $\Ec_i$, we observe that for every agent $i\in\Sc$ and $j\in\Ec_i$, we have
% \begin{align}
%     M_{ij}\leq \frac{dN}{128\frac{N}{2}\frac{d}{4}}=\frac{1}{16}.\nn
% \end{align}

Let $\Sc$ be the set of $\frac{dN}{4}$ pairs $(i,j)\in[N]\times[\frac{d}{2}]$ with smallest $M_{ij}$. From \eqref{eq:ccdN128mainbody} and the definition of $\Sc$, we observe that for every pair $(i,j)\in\Sc$, we have
\begin{align}
    M_{ij}\leq \frac{dN}{64\frac{dN}{4}}=\frac{1}{16}.\nn
\end{align}

Let $B_{ijt}$ be the indicator that the context is such that agent $i$ interacts with $j$-th block in round $t$, which is
\begin{align*}
B_{ijt} = \sind(\x^i_{t,1} = \e_{2j-1}) \,.
\end{align*}
Note that $\{B_{ijt}\}_{t=1}^\infty$ are independent and $\E[B_{ijt}] = 2/d$.
Let $\cT_{ij} = \{t : B_{ijt} = 1\}$ and $\cT^\circ_{ij}$ be the first $T_{\circ}$ elements
of $\cT_{ij}$ with $T_{\circ} = T/(4d)$.
Let
\begin{align}
    R_{ij}=\sum_{t\in \cT^{\circ}_{ij}} \langle\thetab,\x^{i}_{\ast,t}\rangle-\langle\thetab,\x^{i}_{t}\rangle\nn
\end{align}
be
the regret of agent $i$ during the rounds in $\cT_{ij}^\circ$ in bandint instance $\thetab$. 
Note that $\cT_{ij}^\circ$ may contain rounds larger than $T$. Nevertheless,
\begin{align*}
BR_T
&\geq \sum_{i=1}^N \sum_{j=1}^{d/2} \E[R_{ij} \sind(\cT^\circ_{ij} \subset \{1,\ldots,T\})] \\
&\geq \sum_{(i,j) \in \mathcal S}\E[R_{ij} \sind(\cT^\circ_{ij}\subset \{1,\ldots,T\})] \\
&= \sum_{(i,j) \in \mathcal S} \E[R_{ij}] - \E[R_{ij} \sind(\cT^\circ_{ij} \not\subset \{1,\ldots,T\})] \,.
\end{align*}
Suppose that $(i,j) \in \mathcal S$.
Now, $\E[R_{ij}]$ is exactly the Bayesian regret of some policy interacting with the Bayesian two-armed bandit defined
in Lemma~\ref{lemm:singleagentlowerbound} for $T_{\circ}$ rounds. Furthermore, the mutual information between the
optimal action in this bandit and the messages passed to the agent is at most $M_{ij} \leq 1/16$. 
Hence, by Lemma~\ref{lemm:singleagentlowerbound} and Remark~\ref{rem:sa},
\begin{align*}
\E[R_{ij}] \geq \Delta T_{\circ} \left(\frac{1}{2} - \sqrt{\frac{1}{2}\left(\frac{1}{16} + 4 T_{\circ} \Delta^2\right)}\right)\,.
\end{align*}
On the other hand,
\begin{align*}
\E[R_{ij} \sind(\cT_{ij}^\circ \not\subset \{1,\ldots,T\})]
&\leq 2\Delta T_{\circ} \mathbb P(\cT^\circ_{ij} \not\subset \{1,\ldots,T\})= 2\Delta T_{\circ} \mathbb P\left(\sum_{t=1}^T B_{ijt} < T_{\circ}\right).
\end{align*}
By Chernoff's bound, $T\geq 4d\log(8)$ and using $\E[B_{ijt}] = 2/d$, we have
\begin{align*}
2\mathbb P\left(\sum_{t=1}^T B_{ijt} < T_{\circ}\right)
&= 2\mathbb P\left(\sum_{t=1}^T B_{ijt} < T/(4d)\right) \leq 2\exp\left(-T / (4d)\right) \leq \frac{1}{4}\,.
\end{align*}
Therefore, with $\Delta = \frac{1}{8}\sqrt{\frac{d}{T}}$, we have
\begin{align*}
BR_T &\geq \frac{dNT_{\circ}\Delta}{4} \left(\frac{1}{4} - \sqrt{\frac{1}{2}\left(\frac{1}{16} + 4 T_{\circ} \Delta^2\right)} \right)\geq \frac{N\sqrt{dT}}{2500} = \Omega\left(N \sqrt{d T}\right)\,,
\end{align*}
which concludes the proof of Theorem \ref{thm:lowerbound}.

%%%%%%%%%%%%%%%%%%%%%%%%%%%%%%%%%%%%%%%%%%%%%%%%%%%%%%%%%%%%%%%%%%%%%%%%%%%%%%%%%%%%%%%%%%%%%%%%%%%%%%%%%%%%%%%%%%%%%%%%%%%%%%%%%%%%%%%%%%%%%%%%%%%%%%%%%%%%%%%%%%%%%%%%%%%%%%%%%%%%%%%%%%%%%%%%%%%%%%%%%%%%%%%%%%%%%%%%%%%%%%%%%%%%%%%%%%%%%%%%%%%%%%%%%%%%%%%

\section{AN OPTIMAL DISTRIBUTED CONTEXTUAL LINEAR BANDIT ALGORITHM}\label{sec:FPE}

Following our lower bound on the communication cost in the previous section, we now present an algorithm called, \emph{Distributed Batch Elimination Linear Upper Confidence Bound} (\DisBE), whose communication cost matches the lower bound up to logarithmic factors while achieving an optimal regret rate. \DisBE employs a central server with which the
agents communicate with zero latency. Specifically, the agents can
send \emph{local} updates to the central server, which then aggregates and broadcasts the updated
\emph{global} values of interest (will be specified later). We also discuss a high-level description of \emph{Decentralized Batch Elimination Linear Upper Confidence Bound} (\DecBE), which is a modified version of \DisBE in the absence of a central server, where each agent can only communicate with its \emph{immediate neighbors}. We give details of this modification in Appendix \ref{sec:DPEapp}. 

%%%%%%%%%%%%%%%%%%%%%%%%%%%%%%%%%%%%%%%%%%%%%%%%%%%%%%%%%%%%%%%%%%%%%%%%%%%%%%%%%%%%%%%%%%%%%%%%%%%%%%%%%%%%%%%%%%%%%%%%%%%%%%%%%%%%%%%%%%%%%%%%%%%%%%%%%%%%%%%%%%%%%%%%%%%%%%%%%%%%%%%%%%%%%%%%%%%%%%%%%%%%%%%%%%%%%%%%%%%%%%%%%%%%%%%%%%%%%%%%%%%%%%%%%%%%%%%

\subsection{Overview}\label{sec:overview}
Before describing how \DisBE operates for a every agent $i\in[N]$, we note that all agents run \DisBE concurrently. In \DisBE, the time
steps are grouped into $M$ pre-defined batches by a
grid $\Tc = \{\Tc_0,\Tc_1,\ldots,\Tc_M\}$, where $0=\Tc_0\leq \Tc_1\leq \ldots\leq \Tc_M$, $T\leq \Tc_M$ and $T_m = \Tc_{m}-\Tc_{m-1}$ is the length of batch $m$. Our choice of grid implies that for any $m\geq 3$, we have $T_m = \left(a^{2^{m-1}-1}d^{\frac{1}{2}}/N^{\frac{1}{2}}\right)^{\frac{1}{2^{m-2}}}$. Parameter $a$ is chosen such that $T_M=T$ and $\Tc_M = \sum_{m\in[M]}T_m\geq T_M=T$, and therefore our choice of grid $\Tc$ is valid. At each round $t\in[\Tc_{m-1}+1:\Tc_m]$ during batch $m\in[M]$, agent $i$ first constructs confidence intervals for each action's reward, and the actions whose confidence
intervals completely fall below those of other actions are eliminated. We denote the set of feature vectors associated with the surviving actions by
$\Xc_t^{i(m)}=\cap_{k=0}^{m-1}\Ec\left(\Xc_t^i;(\Lambda_k^i,\thetab_k^i,\beta)\right)$, where
\begin{align}
    \Ec\left(\Xc_t^i;(\Lambda_k^i,\thetab_k^i,\beta)\right)&:=\left\{\x\in \Xc_t^i:\left\langle\thetab_k^i,\x\right\rangle+\beta\norm{\x}_{\left(\Lambda_k^i\right)^{-1}}
   \geq \left\langle\thetab_k^i,\y\right\rangle-\beta\norm{\y}_{\left(\Lambda_k^i\right)^{-1}},~\forall \y\in\Xc_t^i\right\}.\nn
\end{align}

Here, $\{\Lambda_k^i\}_{k=0}^{m-1}$ and $\{\thetab_k^i\}_{k=0}^{m-1}$ are agent $i$'s statistics used in computation of $\Xc_t^{(i)m}$ for $t\in[\Tc_{m-1}+1:\Tc_m]$. They are initialized to $\la I$ and $\mathbf{0}$ and will be updated at the end of each batch (will be specified how shortly). Let $\pi^i_0$ be an arbitrary initial policy used in the first batch. Throughout batch $m\in[M]$, agent $i$ uses the same policy $\pi_{m-1}^i$ to select actions from the surviving actions set. At the end of batch $m\in[M]$, agent $i\in[N]$ sends $\ub^{i}_{m}= \sum_{t=\Tc_{m-1}+1}^{\Tc_{m-1}+T_m/2} \x^{i}_{t}y^{i}_{t}$ to the server who broadcasts $\sum_{i=1}^N\ub_m^i$ to all the agents. At the end of batch $m$, agent $i$ updates policy $\pi_m^i$ (used in the next batch) and components that are key in the construction of the surviving actions set in the next batch. In particular, it updates
\begin{align}
       \Lambda_{m}^i &= \lambda I+\frac{NT_m}{2}\mathbb{E}_{\Xc\sim\Dc^i_{m}}\mathbb{E}_{\x\sim\pi^i_{m-1}(\Xc)}[\x\x^\top],\label{eq:Grammatrix}\\
       \thetab_{m}^i &= \left(\Lambda_{m}^i\right)^{-1} \sum_{j=1}^N \ub^{j}_{m}\label{eq:lse},
   \end{align}
where $\la>0$ is a regularization constant and when conditioned on the first $(m-1)$ batches, $\Dc_m^i$ is the distribution based on which the sets of surviving feature vectors $\Xc_t^{i(m)}$ for all $t\in[\Tc_{m-1}+1:\Tc_m]$ are generated.

\begin{algorithm}[ht]
\caption{\DisBE for agent $i$}
   \label{alg:FPE}
\DontPrintSemicolon
  \KwInput{$N$, $d$, $\delta$, $T$, $M, \lambda$}
   {\bf Initialization:} $a = \sqrt{T}\left(\frac{NT}{d}\right)^{\frac{1}{2(2^{M-1}-1)}}$,  $T_1=T_2=a\sqrt{\frac{d}{N}}$, $T_m=\lfloor a\sqrt{T_{m-1}}\rfloor$, $\thetab_{0}^i=\mathbf{0}$, $\Lambda_0^i=\lambda I$, $\Tc_0=0$, $\Tc_m = \Tc_{m-1}+T_m$, $\la=5\log\left(\frac{4dT}{\delta}\right)$, 
   $\beta = 6\sqrt{\log\left(\frac{2KNT}{\delta}\right)}+\sqrt{\la}$,
arbitrary policy $\pi_0^i$\;
   \For{$m=1,\ldots,M$}{
   \For{$t=\Tc_{m-1}+1,\ldots,\min\{\Tc_m,T\}$}
   {
   $\Xc_t^{i(m)}=\cap_{k=0}^{m-1}\Ec\left(\Xc_t^i;(\Lambda_k^i,\thetab_k^i,\beta)\right)$\;
Play arm $a_{i,t}$ associated with feature vector $\x^{i}_{t}\sim \pi^i_{m-1}\left(\Xc_t^{i(m)}\right)$ and observe $y^{i}_{t}$.\;\label{line:decisionrule}
   }
   Send $\ub^{i}_{m}= \sum_{t=\Tc_{m-1}+1}^{\Tc_{m-1}+T_m/2} \x^{i}_{t}y^{i}_{t}$ to the server.\label{line:sendtoserver}\;
   Receive $ \sum_{j=1}^N \ub^{j}_{m}$ from the server. \label{line:receive}\;
   Compute/construct $\Lambda_{m}^i$ and $\thetab_{m}^i$ as in \eqref{eq:Grammatrix} and \eqref{eq:lse}, respectively, $\Sc_m^i$ as in \eqref{eq:Smi}, and $\pi_m^i=\Exp\left(\frac{2\la}{NT_m},\Sc_m^i\right)$, where $\Exp$ is presented in Appendix \ref{sec:omittedalgorithms}.
}
\end{algorithm}

Statistics $\Lambda_m^i$ and $\thetab_m^i$ are used in defining \emph{new} confidence intervals in Lemma \ref{lemm:confidencesets}. We highlight that a direct use of existing standard confidence intervals in the literature such as the one in \cite{ruan2021linear} would fail to guarantee optimal communication cost $\Otilde(dN)$ and require more communication by a factor of $d$ \footnote{$d^2+d$ values per agent, i.e., $\ub^{i}_{m}$ and $\sum_{t=\Tc_{m-1}+1}^{\Tc_{m-1}+T_m/2}\x^{i}_{t}{\x^{i}_{t}}^\top$.}. Using matrix concentration inequalities, we address this issue by replacing matrix $\la I+\sum_{t=\Tc_{m-1}+1}^{\Tc_{m-1}+T_m/2}\sum_{i=1}^N \x^{i}_{t}{\x^{i}_{t}}^\top$, which would have been used if Algorithm 5 in \cite{ruan2021linear} had been directly extended to a multi-agent one, with $\lambda I + (NT_m/2)\mathbb{E}_{\Xc\sim\Dc^i_{m}}\mathbb{E}_{\x\sim\pi^i_{m-1}(\Xc)}[\x\x^\top]$. This allows agent $i$ to communicate only $d$ values ($\ub^{i}_{m}$) while achieving $\Otilde(\sqrt{dNT})$ regret as will be shown in Theorem \ref{thm:regretandCCfpe}.

As the final step of batch $m$, agent $i$ implements Algorithm \ref{alg:exppolicy} with inputs $\frac{2\la}{NT_m},\Sc_m^i$, where 
\begin{align}
    \Sc_m^i=\left\{\Xc_t^{i(m+1)}\right\}_{t=\Tc_{m-1}+T_m/2+1}^{\Tc_m}.\label{eq:Smi}
\end{align}
\Exp, which is presented in Algorithm \ref{alg:exppolicy} in Appendix \ref{sec:omittedalgorithms} and is inspired by Algorithm 3 in \cite{ruan2021linear}, computes policy $\pi_m^i$ that will be used to select actions from the sets of surviving actions in the next batch. This choice of policy coupled with the definition of $\Lambda_m^i$ in \eqref{eq:Grammatrix} guarantees that at all rounds $t\in[\Tc_1+1:T]$, the length of the longest confidence interval in the surviving sets, which is an upper bound on the instantaneous regret of agent $i$ at round $t$, can be bounded by $\Oc(\sqrt{d/NT})$. This allows us to achieve
the optimal $\Oc(\sqrt{dNT})$-type regret, while other exploration policies, such as the G-optimal design, may result in a $\Oc(d\sqrt{NT})$ regret.

%%%%%%%%%%%%%%%%%%%%%%%%%%%%%%%%%%%%%%%%%%%%%%%%%%%%%%%%%%%%%%%%%%%%%%%%%%%%%%%%%%%%%%%%%%%%%%%%%%%%%%%%%%%%%%%%%%%%%%%%%%%%%%%%%%%%%%%%%%%%%%%%%%%%%%%%%%%%%%%%%%%%%%%%%%%%%%%%%%%%%%%%%%%%%%%%%%%%%%%%%%%%%%%%%%%%%%%%%%%%%%%%%%%%%%%%%%%%%%%%%%%%%%%%%%%%%%%

\subsection{Theoretical Results for \DisBE}\label{sec:main result}

We present our theoretical results for \DisBE, which show that it is nearly minimax optimal in terms of \emph{both regret and communication cost}. The proof is given in Appendix \ref{sec:proofofregretandCC}.

\begin{theorem}\label{thm:regretandCCfpe}
Fix $M = 1+\log\left(\log\left(NT/d\right)/2+1\right)$ in Algorithm \ref{alg:FPE}. Suppose Assumption \ref{assum:boundedness} holds. If $T\geq \Omega\left(d^{22}\log^2(NT/\delta)\log^2 d\log^2(d\la^{-1})\right)$, then with probability at least $1-2\delta$, it holds that
$R_T\leq\Oc\left(\sqrt{dNT\log d \log^2\left(KNT/\delta\la\right)}\log\log\left(NT/d\right)\right)$, and $\text{Communication Cost}\leq \Oc\left( dN\log\log\left(NT/d\right)\right)$, where the communication cost is measured by the number of real numbers communicated by the agents.
\end{theorem}

We remark that simple tricks may significantly reduce the exponent constant in constraint $T\geq d^{\Oc(1)}$.  For example, first running a simpler version of \DisBE, in which the exploration policy is the G-optimal design $\piG(\Xc_t^{i(m)})$, for $\sqrt{T/dN}$ rounds and then switching to \DisBE would reduce the exponent to 10.

\begin{remark}\label{rem:bits}
 For the sake of Algorithm \ref{alg:FPE}'s presentation, we find it instructive to consider the communication cost as the number of real numbers communicated in the network. However, it is more realistic if we translate it into the total number of communicated bits. It would also allow us to make a fair comparison with the lower bound in Theorem \ref{thm:lowerbound} as it is stated in terms of number of communicated bits. Therefore, if we slightly modify Algorithm \ref{alg:FPE} such that instead of communicating vectors $\ub_m^i$ in Line \ref{line:sendtoserver}, agent $i$ first rounds each entry of $\ub_m^i$ with precision $\epsilon_0$ and then sends the rounded vector to the server, then $\Oc\left(\log(1/\epsilon_0)\right)$ number of bits is
sufficient to communicate each entry of the rounded vectors $\ub_m^i$. Our analysis in Appendix \ref{sec:finiteprecision} shows that compared to bounds in Theorem \ref{thm:regretandCCfpe}, by selecting $\epsilon_0 = \Oc((1/N\sqrt{dT}))$, the communication cost of this slightly modified version of \DisBE, which is measured in bits, is $\Oc\left( dN\log\log\left(NT/d\right)\log(dNT)\right)$ and its regret is same as \DisBE's.
\end{remark}

% \begin{remark}\label{remark:relaxD}
% As mentioned in Section \ref{sec:overview}, a direct use of the confidence intervals in \citet{ruan2021linear}  would fail to guarantee optimal communication cost $\tilde{\mathcal{O}}(dN)$ and require more communication by a factor of $d$. Therefore, we introduce new confidence intervals in Lemma \ref{lemm:confidencesets} so that \DisBE would enjoy an optimal communication rate. The assumption on the knowledge of $\Dc$ is required in the computation of $\Lambda_m^i$ in \eqref{eq:Grammatrix} used in these new confidence intervals. However, in practice, distribution $\Dc$ is not fully known and/or $\Lambda_m^i$ cannot be computed without any error. In this remark, we relax this assumption and consider more realistic settings where each agent $i$ can estimate matrix $\Lambda_m^i$ in batch $m$ up to an $\epsilon$ error, i.e., $(1-\epsilon)\Lambda_m^i\leq \tilde \Lambda_m^i\leq (1+\epsilon)\Lambda_m^i$, where $\tilde\Lambda_m^i$ is an estimation of $\Lambda_m^i$ and the inequalities hold element-wise. In Appendix \ref{sec:relaxingD}, we show that this error would result in an additional $e^{\frac{2\epsilon d}{1-\epsilon}}$ multiplicative factor in the regret bound. We note that for sufficiently small values of $\epsilon\leq 1/(2d+1)$, this multiplicative factor is a constant.
% \end{remark}

\begin{remark}\label{remark:relaxD}
As mentioned in Section \ref{sec:overview}, a direct use of confidence intervals in \citet{ruan2021linear}  would fail to guarantee optimal communication cost $\tilde{\mathcal{O}}(dN)$ and require more communication by a factor of $d$. Therefore, we use new confidence intervals (see Lemma \ref{lemm:confidencesets}) so that \DisBE would enjoy an optimal communication rate. The assumption on the knowledge of $\Dc$ is required in the computation of $\Lambda_m^i$ in \eqref{eq:Grammatrix} used in these new confidence intervals. However, in practice, distribution $\Dc$ is not fully known and can only be estimated; therefore, $\Lambda_m^i$ cannot be computed without any error. We relax this assumption and consider more realistic settings where each agent $i$ can estimate matrix $\Lambda_m^i$ in batch $m$ up to an $\epsilon_m$ error, i.e., $(1-\epsilon_m)\Lambda_m^i\preceq \tilde \Lambda_m^i\preceq (1+\epsilon_m)\Lambda_m^i$, where $\tilde\Lambda_m^i$ is an estimation of $\Lambda_m^i$ and $\epsilon_m\in(0,1)$\footnote{This is weaker than the component-wise assumption $(1-\epsilon_m)\Lambda_m^i\leq \tilde \Lambda_m^i\leq (1+\epsilon_m)\Lambda_m^i$.}. In Appendix \ref{sec:relaxingD}, we show that for sufficiently small values of $\epsilon_m\leq 1/\sqrt{NT_m}$, a multiplicative factor $(1-\max_{m\in[M]}\epsilon_m)^{-1}$ appears in the regret bound while the communication cost remains unchanged.
\end{remark}

%%%%%%%%%%%%%%%%%%%%%%%%%%%%%%%%%%%%%%%%%%%%%%%%%%%%%%%%%%%%%%%%%%%%%%%%%%%%%%%%%%%%%%%%%%%%%%%%%%%%%%%%%%%%%%%%%%%%%%%%%%%%%%%%%%%%%%%%%%%%%%%%%%%%%%%%%%%%%%%%%%%%%%%%%%%%%%%%%%%%%%%%%%%%%%%%%%%%%%%%%%%%%%%%%%%%%%%%%%%%%%%%%%%%%%%%%%%%%%%%%%%%%%%%%%%%%%%

\subsection{Proof Sketch of Theorem \ref{thm:regretandCCfpe}} 
We first introduce the following lemma that constructs confidence intervals for the expected rewards.
\begin{lemma}[Confidence intervals for \DisBE]\label{lemm:confidencesets}
Suppose Assumption \ref{assum:boundedness} holds. For $\delta\in(0,1)$, let 
    $\beta=6\sqrt{\log\left(\frac{2KNT}{\delta}\right)}+\sqrt{\la}$.
Then for all $\x\in\Xc_t^i,i\in[N],t\in[T], m\in[M]$, with probability at least $1-\delta$, it holds that $\abs{\left\langle\x,\thetab_m^i-\thetab_\ast\right\rangle}\leq \beta\norm{\x}_{\left(\Lambda_m^i\right)^{-1}}$.
\end{lemma}
We prove this lemma by first employing appropriate matrix concentration inequalities to lower bound $\Lambda_m^i$ by matrix $\frac{1}{2}\sum_{t=\Tc_{m-1}+1}^{\Tc_{m-1}+T_m/2}\sum_{i=1}^N \x^{i}_{t}{\x^{i}_{t}}^\top$. Carefully replacing $\Lambda_m^i$ with its lower bound and using Azuma's inequality, we establish confidence intervals stated in the lemma. We highlight that this lemma is key in ensuring an optimal  communication rate $\Otilde(dN)$, as a direct use of confidence intervals in \cite{ruan2021linear} would fail to guarantee optimal communication cost and require $\Otilde(d^2N)$ communication. See Appendix \ref{sec:proofofconfidencesets} for the complete proof.

Thanks to our choice of $T_1$ and $T_2$, and the fact that expected value of the rewards are bounded in $[-1,1]$, the regret of first two batches is bounded by $\Oc(\sqrt{dNT})$. For each batch $m\geq3$, the confidence intervals imply that for all $t\in[\Tc_{m-1}+1:\Tc_m]$,  $\x_{t,\ast}^i\in\Xc_t^{i(m)}$ with high probability, and allow us to bound the instantaneous regret $r^{i}_{t}=\mathbb{E}[ \langle\thetab_\ast,\x^{i}_{\ast,t}\rangle-\langle\thetab_\ast,\x^{i}_{t}\rangle]$ by $4\beta \mathbb{E}_{\Xc\sim\Dc^i_{m-1}}[\max_{\x\in\Xc}\sqrt{\x^\top\left( \Lambda_{m-1}^{i}\right)^{-1}\x}]$. Fortunately, policy $\pi^i_{m-2}$ and Theorem 5 in \cite{ruan2021linear} guarantee that $\mathbb{E}_{\Xc\sim\Dc^i_{m-1}}[\max_{\x\in\Xc}\sqrt{\x^\top\left( \Lambda_{m-1}^{i}\right)^{-1}\x}]$ is bounded by $\Otilde\left(\sqrt{d/(NT_{m-1})}\right)$. Finally, these combined with our choice of grid $\Tc = \{\Tc_0,\Tc_1,\ldots,\Tc_M\}$ and $M = 1+\log\left(\log\left(NT/d\right)/2+1\right)$ lead us to a regret bound $\Otilde(\sqrt{dNT})$. Moreover, communications happen only at the end of each batch, whose number is $M$, and agents only share $d$-dimensional vectors $\ub^{i}_{m}$. Therefore, communication cost is $dNM = \Oc\left( dN\log\log\left(NT/d\right)\right)$.

%%%%%%%%%%%%%%%%%%%%%%%%%%%%%%%%%%%%%%%%%%%%%%%%%%%%%%%%%%%%%%%%%%%%%%%%%%%%%%%%%%%%%%%%%%%%%%%%%%%%%%%%%%%%%%%%%%%%%%%%%%%%%%%%%%%%%%%%%%%%%%%%%%%%%%%%%%%%%%%%%%%%%%%%%%%%%%%%%%%%%%%%%%%%%%%%%%%%%%%%%%%%%%%%%%%%%%%%%%%%%%%%%%%%%%%%%%%%%%%%%%%%%%%%%%%%%%%

\subsection{Decentralized Batch Elimination LUCB without Server}\label{sec:decentralized}
In a scenario where there is no server and the agents are allowed to communicate \emph{only} with their immediate neighbors, they can be represented by nodes of a graph. Applying a carefully designed consensus procedure that guarantees sufficient information mixing among the entire network, in Appendix \ref{sec:DPEapp}, we propose a fully decentralized version of \DisBE, called \DecBE. Communication cost of \DecBE is greater than \DisBE's by an extra multiplication factor $\log(dN)\delta_{\rm max}/\sqrt{1/\abs{\la_2}}$, where $\delta_{\rm max}$ is the maximum degree of the network's graph and $\abs{\la_2}$ is the second largest eigenvalue of the communication matrix in absolute value characterizing the graph's connectivity level. This is because at the last $\log(dN)/\sqrt{1/\abs{\la_2}}$ rounds of each batch $m$, agents communicate each entry of their estimations of vector  $\sum_{j=1}^N\ub_{m}^j$ with their neighbors, whose number is at most $\delta_{\rm max}$, to ensure enough information mixing. Moreover, this results in \DecBE having no control over the regret of the mixing rounds, and therefore an additional term $\log(dN)NM/\sqrt{1/\abs{\la_2}}$, which we call \emph{delay effect}, in the regret bound. The theoretical guarantees of \DecBE are summarized in table \ref{table:comp} and a detailed discussion is given in Appendix \ref{sec:DPEapp}.

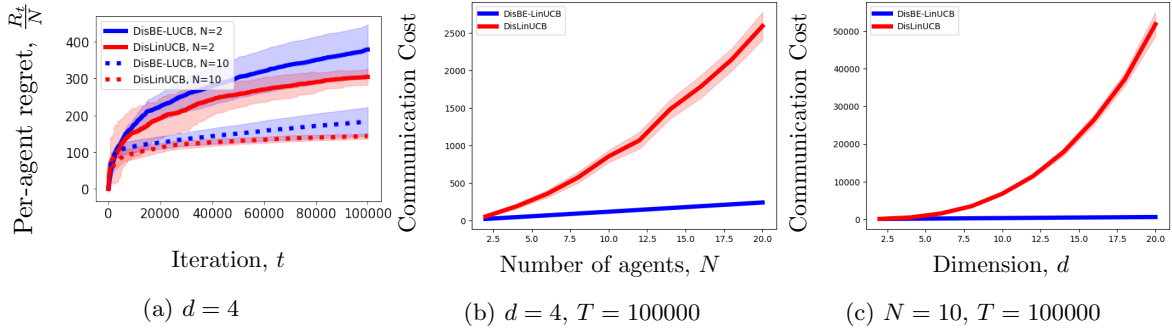
\begin{figure*}[t!]
\centering
\begin{subfigure}{2in}
\begin{tikzpicture}
\node at (0,0) {\includegraphics[scale=0.32]{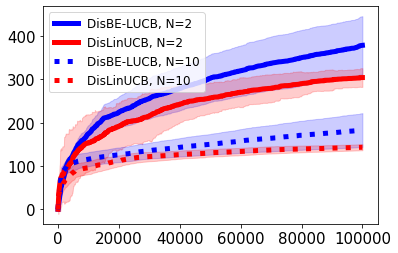}};
\node at (-2.7,0) [rotate=90,scale=1.]{Per-agent regret, $\frac{R_t}{N}$};
\node at (0,-1.9) [scale=0.9]{Iteration, $t$};
\end{tikzpicture}
\caption{$d=4$}
\label{fig:regretcomparison}
\end{subfigure}
\centering
\begin{subfigure}{2in}
\begin{tikzpicture}
\node at (0,0) {\includegraphics[scale=0.32]{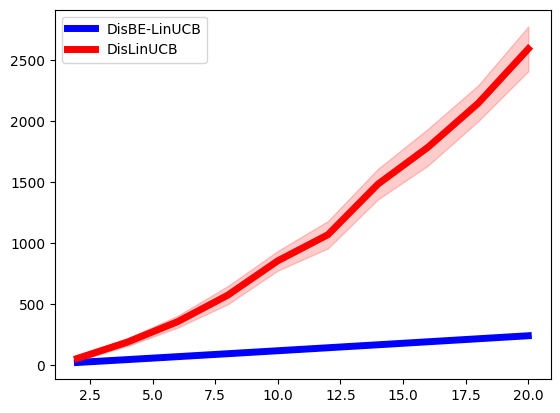}};
\node at (-2.7,0) [rotate=90,scale=0.9]{Communication Cost};
\node at (0,-1.9) [scale=0.9]{Number of agents, $N$};
\end{tikzpicture}
\caption{$d=4$, $T=100000$}
\label{fig:commcostcompareNvarying}
\end{subfigure}
\centering
\begin{subfigure}{2in}
\begin{tikzpicture}
\node at (0,0) {\includegraphics[scale=0.32]{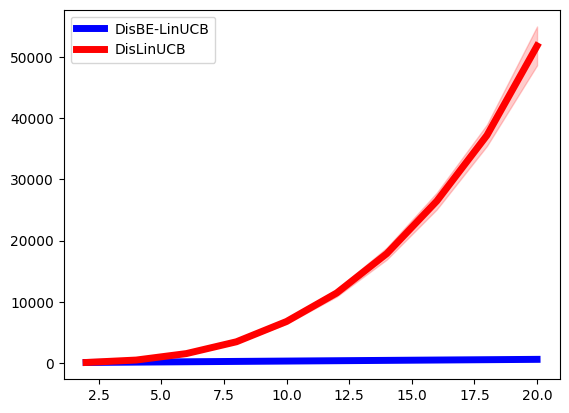}};
\node at (-2.7,0) [rotate=90,scale=0.9]{Communication Cost};
\node at (0,-1.9) [scale=0.9]{Dimension, $d$};
\end{tikzpicture}
\caption{$N=10$, $T=100000$}
\label{fig:commcostcomparedvarying}
\end{subfigure}
\caption{Regret and Communication Cost Comparison. The shaded regions show one standard deviation around the mean. The results are averages over 20 problem realizations.}
\label{fig:comparison}
\end{figure*}

%%%%%%%%%%%%%%%%%%%%%%%%%%%%%%%%%%%%%%%%%%%%%%%%%%%%%%%%%%%%%%%%%%%%%%%%%%%%%%%%%%%%%%%%%%%%%%%%%%%%%%%%%%%%%%%%%%%%%%%%%%%%%%%%%%%%%%%%%%%%%%%%%%%%%%%%%%%%%%%%%%%%%%%%%%%%%%%%%%%%%%%%%%%%%%%%%%%%%%%%%%%%%%%%%%%%%%%%%%%%%%%%%%%%%%%%%%%%%%%%%%%%%%%%%%%%%%%
\section{EXPERIMENTS}\label{sec:experiments}
In this section, we present numerical simulations to confirm our theoretical findings. We evaluate the performance of \DisBE on synthetic data and compare it to that of DisLinUCB proposed by \cite{wang2019distributed} that study the most similar setting to ours. The results shown in Figure \ref{fig:comparison} depict averages over 20 realizations, for which we have chosen $K=20$ and $\delta=0.01$. For each realization, the parameter $\thetab_\ast$ is drawn from $\mathcal{N}(0,I_{d})$ and then normalized to unit norm and noise variables are zero-mean Gaussian random variables with variance $0.01$. The decision set distribution $\Dc$ is chosen to be uniform over $\left\{\tilde\Xc_1,\tilde\Xc_2,\ldots,\tilde\Xc_{100}\right\}$, where each $\tilde\Xc_i$ is a set of $K$ vectors drawn from $\mathcal{N}(0,I_{d})$ and then normalized to unit norm. While implementing \DisBE, in order to compute $\mathbb{E}_{\Xc\sim\Dc^i_{m}}\mathbb{E}_{\x\sim\pi^i_{m-1}(\Xc)}[\x\x^\top]$ for agent $i$ at batch $m$, we followed these steps: 1) for each $j\in[100]$, we built $\tilde\Xc_j^{i(m)}=\cap_{k=0}^{m-1}\Ec\left(\tilde\Xc_j;(\Lambda_k^i,\thetab_k^i,\beta)\right)$; 2) we took average over all 100 matrices $\frac{1}{100}\sum_{j\in[100]}\mathbb{E}_{\x\sim\pi^i_{m-1}(\tilde\Xc_j^{i(m)})}[\x\x^\top]$ as $\Dc$ is a uniform distribution over $\left\{\tilde\Xc_1,\tilde\Xc_2,\ldots,\tilde\Xc_{100}\right\}$. In Figure \ref{fig:regretcomparison}, fixing $d=4$, we compare the per-agent regret $R_t/N$ of \DisBE and DisLinUCB for $t\in[100000]$ and for different values of $N=2$ and $N=10$, where $R_t= \sum_{s=1}^t\sum_{i=1}^N\langle\thetab_\ast,\x^i_{\ast,s}\rangle-\langle\thetab_\ast,\x^i_{s}\rangle$\footnote{In order to have a fair comparison with DisLinUCB, the cumulative regret depicted in Figure \ref{fig:regretcomparison} does not include the expectation as defined in \eqref{eq:cumulativeregret}.}. Figure \ref{fig:commcostcompareNvarying} compares the communication cost of \DisBE and DisLinUCB when both algorithms are implemented for fixed $d =4$ and $T=100000$, and $N$ varying from $2$ to $20$. Finally, Figure \ref{fig:commcostcomparedvarying} compares the communication cost of \DisBE and DisLinUCB when both algorithms are implemented for fixed $N=10$ and $T=100000$, and $d$ varying from $2$ to $20$. From these three comparisons, we conclude that \DisBE achieves a regret comparable with DisLinUCB, at a significantly smaller communication rate. The curves in Figures \ref{fig:commcostcompareNvarying} and \ref{fig:commcostcomparedvarying} verify the linear dependency of \DisBE's communication cost on $N$ and $d$ while communication cost of DisLinUCB grows super-linearly with $N$ and $d$ (see Table \ref{table:comp} for theoretical comparisons). Moreover, Figure \ref{fig:regretcomparison} emphasizes the value of collaboration in speeding up the learning process. As the number of agents increases, each agent learns the environment faster as an individual.

%%%%%%%%%%%%%%%%%%%%%%%%%%%%%%%%%%%%%%%%%%%%%%%%%%%%%%%%%%%%%%%%%%%%%%%%%%%%%%%%%%%%%%%%%%%%%%%%%%%%%%%%%%%%%%%%%%%%%%%%%%%%%%%%%%%%%%%%%%%%%%%%%%%%%%%%%%%%%%%%%%%%%%%%%%%%%%%%%%%%%%%%%%%%%%%%%%%%%%%%%%%%%%%%%%%%%%%%%%%%%%%%%%%%%%%%%%%%%%%%%%%%%%%%%%%%%%%

\section{CONCLUSION}\label{sec:conclusion}
We proved an information-theoretic lower bound on the communication cost of any algorithm achieving an optimal regret rate for the distributed contextual linear bandit problem with stochastic contexts. 
We then proposed \DisBE with regret $\Otilde(\sqrt{dNT})$ and communication cost $\Otilde(dN)$ which (nearly) matches our lower bound and improves upon the previous best-known algorithms whose communication cost scale super linearly either in $d$ or $N$. This showed that \DisBE is nearly minimax optimal in terms of \emph{both regret and communication cost}. Finally, we proposed \DecBE, a fully decentralized variant of \DisBE,  without a central server where the agents can only communicate with their immediate neighbors given by a communication graph. We showed that the structure of the network affects the regret performance via a small additive term that depends on the spectral gap of the underlying graph, while the communication cost still grows linearly with $d$ and $N$. As shown in Table \ref{table:comp}, the best communication cost achieved for settings with \emph{adversarially} varying contexts over time horizon and agents is of order $\Oc(d^3N^{1.5})$. There is no formal theory proving such bounds are optimal for the adversarial context case. While our work provides optimal theoretical guarantees for stochastically varying contexts, it is not clear how to generalize these \emph{optimal} results to settings with adversarially varying contexts. Therefore, an important future direction is to design optimal algorithms and prove communication cost lower bounds for scenarios with adversarial contexts.

\newpage
\bibliographystyle{apalike}
\bibliography{references}

\newpage
\appendix
\onecolumn

% \section{PROOF OF THEOREM \ref{thm:lowerbound}}\label{sec:proofoflowerboubd}

\section{PROOF OF LEMMA \ref{lemm:singleagentlowerbound}}\label{sec:proofofmutual}
Let $\mub\sim\Unif(\mub_1,\mub_2)$, where $\mub_1=[\Delta,0]^\top, \mub_2 = [-\Delta,0]^\top$, $\z=\{z_t\}_{t=1}^T$ be the set of arm 1's reward, $H=\{a_t,y_t\}_{t=1}^T$ be the history over the course of $T$ rounds, where $a_t$ is the arm pulled and $y_t$ is the observed reward at round $t$, $a_\ast=\argmax_{a\in\{1,2\}}\mub_a$, and $\hat a\sim\Unif(\{a_1,a_2,\ldots,a_T\})$. We have
\begin{align}
    \mathbb{E}[R_T(\pi,\mub)]&=\mathbb{E}[\sum_{t=1}^T\mathbbm{1}(\hat a\neq a_\ast)\Delta]\nn\\
    &=\Delta T \mathbb{P}(\hat a\neq a_\ast)\tag{$\bigstar$}.
\end{align}
Now, we lower bound $\mathbb{P}(\hat a\neq a_\ast)$ as follows
\begin{align}
    \mathbb{P}(\hat a\neq a_\ast)&=\sum_{a\in\{1,2\}}\mathbb{P}(a_\ast=a)\mathbb{P}(\hat a\neq a\vert a_\ast=a)\nn\\
    &= \sum_{a\in\{1,2\}}\mathbb{P}(a_\ast=a)\left[\mathbb{P}(\hat a\neq a)+\mathbb{P}(\hat a= a)-\mathbb{P}(\hat a= a\vert a_\ast=a)\right]\nn\\
    &\geq \sum_{a\in\{1,2\}}\mathbb{P}(a_\ast=a)\left[\mathbb{P}(\hat a\neq a)-\sqrt{\frac{1}{2}\DKL(\mathbb{P}_{\hat a\vert a_\ast=a},\mathbb{P}_{\hat a})}\right]\tag{Pinsker's inequality}\\
    &=\frac{1}{2}-\sum_{a\in\{1,2\}}\mathbb{P}(a_\ast=a)\sqrt{\frac{1}{2}\DKL(\mathbb{P}_{\hat a\vert a_\ast=a},\mathbb{P}_{\hat a})}\nn\\
    &\geq\frac{1}{2}-\sqrt{\frac{1}{2}\sum_{a\in\{1,2\}}\mathbb{P}(a_\ast=a)\DKL(\mathbb{P}_{\hat a\vert a_\ast=a},\mathbb{P}_{\hat a})}\tag{Jensen's inequality}\nn\\
    &=\frac{1}{2}-\sqrt{\frac{1}{2}I(\hat a;a_\ast)}\nn\\
    &\geq \frac{1}{2}-\sqrt{\frac{1}{2}I(M,H;a_\ast)}\tag{Data processing}\\
    &\geq \frac{1}{2}-\sqrt{\frac{1}{2}\left(I(M;a_\ast)+I(H;a_\ast)\right)}\nn\\
    &\geq \frac{1}{2}-\sqrt{\frac{1}{2}\left(\frac{1}{16}+I(H;a_\ast)\right)}\tag{$\bigstar\bigstar$}.
\end{align}
In our next step towards lower bounding $\mathbb{P}(\hat a\neq a_\ast)$, we upper bound $I(H;a_\ast)$, as follows
\begin{align}
    I(H;a_\ast)&\leq  I(\z;a_\ast)\tag{Data processing}\\
    &=\sum_{a\in\{1,2\}}\frac{1}{2}\DKL(\mathbb{P}(\z\vert a_\ast=a),\mathbb{P}(\z))\nn\\
    &\leq \sum_{b\in\{1,2\}}\sum_{a\in\{1,2\}}\frac{1}{2}\DKL\left(\mathbb{P}(\z\vert a_\ast=a),\mathbb{P}(\z\vert a_\ast=b)\right)\nn\\
    &=\frac{1}{2}\DKL\left(\mathbb{P}(\z\vert a_\ast=1),\mathbb{P}(\z\vert a_\ast=2)\right)+\frac{1}{2}\DKL\left(\mathbb{P}(\z\vert a_\ast=2),\mathbb{P}(\z\vert a_\ast=1)\right)\nn\\
    &=\frac{1}{2}\left[T(2\Delta)^2+T(2\Delta)^2\right]\nn\\
    &=4T\Delta^2\tag{$\bigstar\bigstar\bigstar$}.
\end{align}
Combining $\bigstar$, $\bigstar\bigstar$, and $\bigstar\bigstar\bigstar$, we have
\begin{align}
    \mathbb{E}[R_T(\pi,\mub)]\geq \Delta T\left(\frac{1}{2}-\sqrt{\frac{1}{2}\left(\frac{1}{16}+4T\Delta^2\right)}\right),\nn
\end{align}
% Therefore, at least one of $R_T(\pi,\mub_1)$ and $R_T(\pi,\mub_2)$ is greater than $\Delta T\left(\frac{1}{2}-\sqrt{\frac{1}{2}\left(\frac{1}{16}+4T\Delta^2\right)}\right)$,
which concludes the lemma.

\section{PROOF OF THEOREM \ref{thm:regretandCCfpe}}
\label{sec:proofofregretandCC}
\subsection{Proof of Lemma \ref{lemm:confidencesets}}
\label{sec:proofofconfidencesets}

For each batch $m\in[M]$, let $\bb_m=\sum_{t=\Tc_{m-1}+1}^{\Tc_{m-1}+T_m/2}\sum_{i=1}^N \x^{i}_{t}y^{i}_{t}$ and $\Vb_m = \sum_{t=\Tc_{m-1}+1}^{\Tc_{m-1}+T_m/2}\sum_{i=1}^N \x^{i}_{t}{\x^{i}_{t}}^\top$. We have
\begin{align}
    \Lambda_m^i &= \lambda I+\frac{NT_m}{2}\mathbb{E}_{\Xc\sim\Dc^i_{m}}\mathbb{E}_{\x\sim\pi_{m-1}^i(\Xc)}[\x\x^\top]\nn\\
    &=\la I +\frac{NT_m}{4}\left(2\mathbb{E}_{\Xc\sim\Dc^i_{m}}\mathbb{E}_{\x\sim\pi_{m-1}^i(\Xc)}[\x\x^\top]+6\gamma I\right)-1.5NT_m\gamma I.\label{eq:firstlowerboundonLambda_m^i}
\end{align}
% Let $\gamma = \frac{3\log(\frac{2d}{\delta})}{NT_m}$. Thus, 

By choosing $\gamma = \frac{3\log(\frac{4dT}{\delta})}{NT_m}$  and $\la = 5\log\left(\frac{4dT}{\delta}\right)$, combining \eqref{eq:firstlowerboundonLambda_m^i} and Lemma \ref{lemm:upperboundonrandommatrixwithcuttoff}, for all $m\in[M]$, with probability at least $1-\delta/2$, we have
\begin{align}
     \Lambda_m^i &\succeq \left(\la-5\log\left(\frac{4dT}{\delta}\right)\right) I +\frac{1}{2}\sum_{t=\Tc_{m-1}+1}^{\Tc_{m-1}+T_m/2}\sum_{i=1}^N \x^{i}_{t}{\x^{i}_{t}}^\top\nn\\
     &=\frac{1}{2}\Vb_m.\label{eq:lowerboundonLambdam}
\end{align}
 Moreover, for a fixed $\x\in\Xc_t^i$ and $(i,t)\in[N]\times[T]$, let $z_{t,m}^{j,i}=\x^\top\left(\Lambda_m^i\right)^{-1}\left(\x_t^jy_t^j-\mathbb{E}_{\Xc\sim\Dc^i_{m}}\mathbb{E}_{\x\sim\pi_{m-1}^i(\Xc)}[\x\x^\top]\thetab_\ast\right)$. Thus, we have
\begin{align}
  \abs{\left\langle\x,\thetab_m^i-\thetab_\ast\right\rangle}&=\abs{\left\langle\x,\left(\Lambda_m^i\right)^{-1}\bb_m-\thetab_\ast\right\rangle}\nn\\
    &=\abs{\left\langle\x,\left(\Lambda_m^i\right)^{-1}\bb_m\right\rangle-\left\langle\x,\left(\Lambda_m^i\right)^{-1}\Lambda_m^i\thetab_\ast\right\rangle}\nn\\
   &\leq\abs{\left\langle\x,\left(\Lambda_m^i\right)^{-1}\bb_m\right\rangle-\left\langle\x,\left(\Lambda_m^i\right)^{-1}\left(\Lambda_m^i-\la I\right)\thetab_\ast\right\rangle}+\abs{\lambda\langle \x,\left(\Lambda_m^i\right)^{-1}\thetab_\ast\rangle}\nn\\
   &\leq \abs{\x^\top\left(\Lambda_m^i\right)^{-1}\left(\bb_m-\frac{NT_m}{2}\mathbb{E}_{\Xc\sim\Dc^i_{m}}\mathbb{E}_{\x\sim\pi_{m-1}^i(\Xc)}[\x\x^\top]\thetab_\ast\right)}+\sqrt{\la}\norm{\x}_{\left(\Lambda_m^i\right)^{-1}}\tag{Cauchy Schwarz inequality and Assumption \ref{assum:boundedness}}\\
   &=\abs{\sum_{t=\Tc_{m-1}+1}^{\Tc_{m-1}+T_m/2}\sum_{j=1}^Nz_{t,m}^{j,i}}+\sqrt{\la}\norm{\x}_{\left(\Lambda_m^i\right)^{-1}}\nn.
\end{align}

Note that
\begin{align}
    \mathbb{E}\left[z_{t,m}^{j,i}\right] = \mathbb{E}\left[\x^\top\left(\Lambda_m^i\right)^{-1}\left(\x_t^j({\x_t^j}^\top\thetab_\ast+\eta_t^j)-\mathbb{E}_{\Xc\sim\Dc^i_{m}}\mathbb{E}_{\x\sim\pi_{m-1}^i(\Xc)}[\x\x^\top]\thetab_\ast\right)\right] = 0 \tag{Noise $\eta_t^j$ is zero-mean and independent of $\x_t^j$},
\end{align}
By
Azuma’s inequality, for a fixed $\x\in\Xc_t^i$ and $(i,t)\in[N]\times[T]$, we have 
\begin{align}
    \mathbb{P}\left(\abs{\sum_{t=\Tc_{m-1}+1}^{\Tc_{m-1}+T_m/2}\sum_{j=1}^Nz_{t,m}^{j,i}}\geq \alpha\norm{\x}_{\left(\Lambda_m^i\right)^{-1}}\right)&\leq 2{\rm exp}\left(\frac{-\alpha^2\norm{\x}_{\left(\Lambda_m^i\right)^{-1}}^2}{2c_m^i}\right),\label{eq:azuma}
\end{align}
where 
\begin{align}
    c_m^i &= \sum_{t=\Tc_{m-1}+1}^{\Tc_{m-1}+T_m/2}\sum_{j=1}^N\abs{\x^\top\left(\Lambda_m^i\right)^{-1}\left(\x_t^jy_t^j-\mathbb{E}_{\Xc\sim\Dc^i_{m}}\mathbb{E}_{\x\sim\pi_{m-1}^i(\Xc)}[\x\x^\top]\thetab_\ast\right)}^2\nn\\
    &\leq 2\sum_{t=\Tc_{m-1}+1}^{\Tc_{m-1}+T_m/2}\sum_{j=1}^N\abs{\x^\top\left(\Lambda_m^i\right)^{-1}\x_t^jy_t^j}^2+NT_m\abs{\x^\top\left(\Lambda_m^i\right)^{-1}\mathbb{E}_{\Xc\sim\Dc^i_{m}}\mathbb{E}_{\x\sim\pi_{m-1}^i(\Xc)}[\x\x^\top]\thetab_\ast}^2\nn\\
    &\leq 2\sum_{t=\Tc_{m-1}+1}^{\Tc_{m-1}+T_m/2}\sum_{j=1}^N\abs{\x^\top\left(\Lambda_m^i\right)^{-1}\x_t^j}^2+\frac{4}{NT_m}\abs{\x^\top\left(\Lambda_m^i\right)^{-1}\left(\Lambda_m^i-\la I\right)\thetab_\ast}^2\tag{Assumption \ref{assum:boundedness}}\\
    &= 2\sum_{t=\Tc_{m-1}+1}^{\Tc_{m-1}+T_m/2}\sum_{j=1}^N\x^\top\left(\Lambda_m^i\right)^{-1}\x_t^j{\x_t^j}^\top\left(\Lambda_m^i\right)^{-1}\x+\frac{4}{NT_m}\abs{\x^\top\thetab_\ast-\la\x^\top\left(\Lambda_m^i\right)^{-1}\thetab_\ast}^2\nn\\
    &\leq 2\x^\top\left(\Lambda_m^i\right)^{-1}\Vb_m\left(\Lambda_m^i\right)^{-1}\x+\left(\frac{4\norm{\thetab_\ast}_{\Lambda_m^i}^2}{NT_m}+\frac{4\la}{NT_m}\right)\norm{\x}_{\left(\Lambda_m^i\right)^{-1}}^2\tag{Cauchy Schwarz inequality and Assumption \ref{assum:boundedness}}\\
    &\leq 4\x^\top\left(\Lambda_m^i\right)^{-1}\Lambda_m^i\left(\Lambda_m^i\right)^{-1}\x+\left(\frac{4\norm{\thetab_\ast}_{\Lambda_m^i}^2}{NT_m}+\frac{4\la}{NT_m}\right)\norm{\x}_{\left(\Lambda_m^i\right)^{-1}}^2\tag{Conditioned on the event in Eqn. \eqref{eq:lowerboundonLambdam}}\\
    &\leq\left(6+\frac{8\la}{NT_m}\right)\norm{\x}_{\left(\Lambda_m^i\right)^{-1}}^2,\label{eq:upperboundoncm}
\end{align}
where the last inequity follows from the fact that 

% \begin{align}
%     \norm{\x}_{\left(\Lambda_m^i\right)^{-1}}^2\geq \frac{\norm{\x}^2}{\lamax\left(\Lambda_m^i\right)}\geq \frac{2H^2}{NT_m}.
% \end{align}

\begin{align}
\norm{\thetab_\ast}_{\Lambda_m^i}^2\leq \norm{\thetab_\ast}_2^2\lamax\left(\Lambda_m^i\right)\leq \la+\frac{NT_m}{2}\tag{Assumption \ref{assum:boundedness}}.
\end{align}
Combining \eqref{eq:azuma} and \eqref{eq:upperboundoncm}, and by a union bound, we have
\begin{align}
    \mathbb{P}\left(\abs{\left\langle\x,\thetab_m^i-\thetab_\ast\right\rangle}\leq \left(6\sqrt{\log\left(\frac{2KNT}{\delta}\right)}+\sqrt{\la}\right)\norm{\x}_{\left(\Lambda_m^i\right)^{-1}},~\forall \x\in\Xc_t^i,i\in[N],t\in[T], m\in[M]\right)\geq 1-\delta.
\end{align}

%%%%%%%%%%%%%%%%%%%%%%%%%%%%%%%%%%%%%%%%%%%%%%%%%%%%%%%%%%%%%%%%%%%%%%%%%%%%%%%%%%%%%%%%%%%%%%%%%%%%%%%%%%%%%%%%%%%%%%%%%%%%%%%%%%%%%%%%%%%%%%%%%%%%%%%%%%%%%%%%%%%%%%%%%%%%%%%%%%%%%%%%%%%%%%%%%%%%%%%%%%%%%%%%%%%%%%%%%%%%%%%%%%%%%%%%%%%%%%%%%%%%%%%%%%%%%%%

\subsection{Completing the proof of Theorem \ref{thm:regretandCCfpe}}

Next, we state the following lemma, which we borrow from Theorem 5 in \cite{ruan2021linear} and is used in the proof analysis of Theorem \ref{thm:regretandCCfpe}.

\begin{lemma}[\cite{ruan2021linear}]\label{lemm:mainlemmafromGdistpaper}
Let $\Xc_1,\Xc_2,\ldots,\Xc_L\sim\Dc$ be i.i.d drawn from a distribution $\Dc$ and input of Algorithm \ref{alg:exppolicy} and let $\pi$ be the output policy of Algorithm \ref{alg:exppolicy}. For any $\lambda\in(0,1)$, we have
\begin{align}
    \mathbb{P}\left[\Vlambda^\la_{\Dc}(\pi)\leq \Oc\left(\sqrt{d\log d\log (\la^{-1})}\right)\right]\geq 1-{\rm exp}\left(\Oc(d^3\log d\log(d\la^{-1}))-L d^{-2c}2^{-16}\right),\nn
\end{align}
where we define the $\lambda$-deviation of policy $\pi$ over $\Dc$ by
\begin{align}
    \Vlambda_{\Dc}^\la(\pi):= \mathbb{E}_{\Xc\sim\Dc}\left[\max_{\x\in\Xc}\sqrt{\x^\top \left(\la I+\mathbb{E}_{\Xc\sim\Dc} \mathbb{E}_{\y\sim\pi(\Xc)}[\y\y^\top]\right)^{-1}\x}\right].\label{eq:lambdadeviation}
\end{align}
\end{lemma}
\begin{corollary}
As a direct corollary of Lemma \ref{lemm:mainlemmafromGdistpaper}, if $T\geq\Omega\left(d^{22}\log^2(\frac{NT}{\delta})\log^2 d\log^2(dNT\la^{-1})\right) $, then for all $m\geq 2$ and $i\in[N]$, with probability at least $1-\delta$, it holds that
\begin{align}
    {{\mathbb{V}}}^{(\frac{2\la}{NT_m})}_{\Dc^i_{m}}(\pi^i_{m-1})\leq \Oc(\sqrt{d\log d\log(NT\la^{-1})}).\label{eq:upperboundonvlambdapi}
\end{align}
\end{corollary}

Now, we focus on the regret of the $i$-th agent at $m$-th batch for any $m\geq 3$. Let $\Dc_m^i$ be the distribution based on which the surviving sets $\Xc_t^{i(m)}$ for all $t\in[\Tc_{m-1}+1:\Tc_m]$ are generated when conditioned on the first $m-1$ batches. For any $t\in[\Tc_{m-1}+1:\Tc_m]$, conditioned on the event that the confidence intervals in Lemma \ref{lemm:confidencesets} hold, we have
\begin{align}
    r^{i}_{t} &=\mathbb{E}\left[ \langle\thetab_\ast,\x^{i}_{\ast,t}\rangle-\langle\thetab_\ast,\x^{i}_{t}\rangle\right] \nn\\
    &\leq \mathbb{E}\left[ \langle\thetab_{m-1}^i,\x^{i}_{\ast,t}\rangle-\langle\thetab_{m-1}^i,\x^{i}_{t}\rangle+\beta\norm{\x^{i}_{\ast,t}}_{\left(\Lambda_{m-1}
    ^i\right)^{-1}}+\beta\norm{\x^{i}_{t}}_{\left(\Lambda_{m-1}
    ^i\right)^{-1}}\right] \tag{Lemma \ref{lemm:confidencesets}}\\
     &\leq 2\beta \mathbb{E}\left[\norm{\x^{i}_{\ast,t}}_{\left(\Lambda_{m-1}
    ^i\right)^{-1}}+\norm{\x^{i}_{t}}_{\left(\Lambda_{m-1}
    ^i\right)^{-1}}\right] \tag{$\x^{i}_{\ast,t}\in\Xc_t^{i(m)}$}\\
     &\leq 4\beta \mathbb{E}\left[\max_{\x\in\Xc_t^{i(m)}}\norm{\x}_{\left(\Lambda_{m-1}
    ^i\right)^{-1}}\right]\nn\\
     &\leq 4\beta \mathbb{E}_{\Xc\sim\Dc_m^i}\left[\max_{\x\in\Xc}\norm{\x}_{\left(\Lambda_{m-1}
    ^i\right)^{-1}}\right]\nn\\
     &\leq 4\beta \mathbb{E}_{\Xc\sim\Dc^i_{m-1}}\left[\max_{\x\in\Xc}\norm{\x}_{\left(\Lambda_{m-1}
    ^i\right)^{-1}}\right]\nn\\
     &\leq \frac{8\beta}{\sqrt{NT_{m-1}}} \mathbb{E}_{\Xc\sim\Dc^i_{m-1}}\left[\max_{\x\in\Xc}\sqrt{\x^\top \left(\frac{2\la}{NT_{m-1}} I+\mathbb{E}_{\Xc\sim\Dc^i_{m-1}} \mathbb{E}_{\y\sim\pi_{m-2}^i(\Xc)}[\y\y^\top]\right)^{-1}\x}\right]\nn\\
     &=\frac{8\beta}{\sqrt{NT_{m-1}}}{{\mathbb{V}}}^{(\frac{2\la}{NT_{m-1}})}_{\Dc^i_{m-1}}(\pi_{m-2}^i)
     \label{eq:tobecontinuedfpe},
\end{align}
where the third inequality follows from our established confidence intervals in Lemma \ref{lemm:confidencesets} guaranteeing that $\x_{\ast,t}^i\in\Xc_t^{i(m)}$ for all $(i,t,m)\in[N]\times[\Tc_{m-1}+1:\Tc_m]\times[M]$ with probability at least $1-
\delta$. Now, continuing form \eqref{eq:tobecontinuedfpe}, we bound the cumulative regret of batches $m\geq 3$, as follows:
\begin{align}
    \sum_{t=\Tc_2+1}^T\sum_{i=1}^N r^{i}_{t} &\leq  \sum_{m=3}^M  \frac{8\beta N T_m}{\sqrt{NT_{m-1}}}{{\mathbb{V}}}^{(\frac{2\la}{NT_{m-1}})}_{\Dc^i_{m-1}}(\pi_{m-2}^i)\nn\\
    &\leq 8 \beta \sqrt{dN\log d\log (NT\la^{-1})} \sum_{m=2}^M \frac{T_m}{\sqrt{T_{m-1}}}\tag{Conditioned on the event in Eqn. \eqref{eq:upperboundonvlambdapi}}\\
 &= 8 \beta Ma\sqrt{dN\log d\log (NT\la^{-1})}.\label{eq:regretfpeofbatchesmgreaterthan2fpe}
\end{align}

Next, we bound cumulative regret of the first two batches. Under Assumption \ref{assum:boundedness}, during the first two batches, the instantaneous regret of each agent $i$ at any round $t$ is at most 2. Therefore 
\begin{align}
    \sum_{t=1}^{\Tc_2}\sum_{i=1}^N r^{i}_{t} \leq 2N\Tc_2 = 4a\sqrt{dN}.\label{eq:regretoffirstepochfpe}
\end{align}
Note that for any $m\geq 3$, we can write $T_m$ as
\begin{align}
    T_m = aT_{m-1}^{\frac{1}{2}} = a^{\frac{3}{2}}T_{m-2}^{\frac{1}{4}}=\ldots &= a^{\frac{2^{m-2}-1}{2^{m-3}}}T_2^{\frac{1}{2^{m-2}}}\nn\\
     T_m &=a^{\frac{1}{2^{m-2}}} a^{\frac{2^{m-2}-1}{2^{m-3}}}\left(\frac{T_2}{a}\right)^{\frac{1}{2^{m-2}}}\nn\\
    &=a^{\frac{2^{m-1}-1}{2^{m-2}}}\left(\sqrt{\frac{d}{N}}\right)^{\frac{1}{2^{m-2}}}\nn\\
    &=\left(\frac{a^{2^{m-1}-1}d^{\frac{1}{2}}}{N^{\frac{1}{2}}}\right)^{\frac{1}{2^{m-2}}}.\nn
\end{align}

Our choice of $a$ in the algorithm ensures that for any $M>0$, $T_M = T$ and $\sum_{m=1}^MT_m\geq T_M=T$, and thus the choice of grid $\{\Tc_1,\ldots,\Tc_M\}$ is valid. If we let $M = 1+\log\left(\frac{\log\left(\frac{NT}{d}\right)}{2}+1\right)$, from \eqref{eq:regretfpeofbatchesmgreaterthan2fpe} and \eqref{eq:regretoffirstepochfpe}, we conclude that, with probability at least $1-2\delta$, it holds that
\begin{align}
    R_T &\leq 4\sqrt{dNT}\left(\frac{NT}{d}\right)^{\frac{1}{2(2^{M-1}-1)}}+8 \beta M\sqrt{dNT\log d\log (NT\la^{-1})} \left(\frac{NT}{d}\right)^{\frac{1}{2(2^{M-1}-1)}}\nn\\
    &\leq \Oc\left(\sqrt{dNT\log d \log^2\left(\frac{KNT}{\delta\la}\right)}\log\log\left(\frac{NT}{d}\right)\right).
\end{align}

%%%%%%%%%%%%%%%%%%%%%%%%%%%%%%%%%%%%%%%%%%%%%%%%%%%%%%%%%%%%%%%%%%%%%%%%%%%%%%%%%%%%%%%%%%%%%%%%%%%%%%%%%%%%%%%%%%%%%%%%%%%%%%%%%%%%%%%%%%%%%%%%%%%%%%%%%%%%%%%%%%%%%%%%%%%%%%%%%%%%%%%%%%%%%%%%%%%%%%%%%%%%%%%%%%%%%%%%%%%%%%%%%%%%%%%%%%%%%%%%%%%%%%%%%%%%%%%

\subsection{Communication cost as number of bits transmitted}\label{sec:finiteprecision}
In this section, we consider the number of bits transmitted in a slightly modified version of \DisBE.
To this end, we make the following minor modification to \DisBE. Let $\epsilon_0$ be an additional input to the algorithm. In Line \ref{line:receive} of \DisBE, agent $i$ sends vector $\tilde\ub_m^i$ which is an $\epsilon_0$-precise rounded version of $\ub_m^i$. In particular, if it rounds each entry of $\ub_m^i$ with precision $\epsilon_0$, vector $\tilde\ub_m^i$ will be obtained. Now, we observe how this extra rounding step affects confidence intervals in Lemma \ref{lemm:confidencesets}. In fact, we are interested in upper bounds on $\abs{\left\langle\x,\tilde\thetab_m^i-\thetab_\ast\right\rangle}$, where $\tilde\thetab_{m}^i = \left(\Lambda_{m}^i\right)^{-1} \sum_{i=1}^N \tilde\ub^{i}_{m}$.

For $\delta\in(0,1)$, let 
    $\beta=6\sqrt{\log\left(\frac{2KNT}{\delta}\right)}+\sqrt{\la}$.
Then for all $\x\in\Xc_t^i,i\in[N],t\in[T], m\in[M]$, with probability at least $1-\delta$, it holds that
\begin{align}
\abs{\left\langle\x,\tilde\thetab_m^i-\thetab_\ast\right\rangle}&=\abs{\left\langle\x,\tilde\thetab_m^i-\thetab_m^i+\thetab_m^i-\thetab_\ast\right\rangle}\nn\\
    &\leq \abs{\left\langle\x,\tilde\thetab_m^i-\thetab_m^i\right\rangle}+\abs{\left\langle\x,\thetab_m^i-\thetab_\ast\right\rangle}\nn\\
    &\leq \left(\norm{\tilde\thetab_m^i-\thetab_m^i}_{\Lambda_m^{i}}+\beta\right)\norm{\x}_{\left(\Lambda_m^i\right)^{-1}}\tag{Lemma \ref{lemm:confidencesets} and Cauchy Schwarz inequality}\nn\\
    &\leq \left(\sqrt{\lambda_{\rm max}(\Lambda_m^i)}\norm{\tilde\thetab_m^i-\thetab_m^i}_2+\beta\right)\norm{\x}_{\left(\Lambda_m^i\right)^{-1}}
    \nn\\
    &\leq \left(N\sqrt{dT}\epsilon_0+\beta\right)\norm{\x}_{\left(\Lambda_m^i\right)^{-1}}.
\end{align}
Therefore, letting $\epsilon_0 = \frac{\beta}{N\sqrt{dT}}$, we have 
\begin{align}
\abs{\left\langle\x,\tilde\thetab_m^i-\thetab_\ast\right\rangle}\leq 2\beta\norm{\x}_{\left(\Lambda_m^i\right)^{-1}},
\end{align}
which implies that replacing $\beta$ in \DisBE with $2\beta$, will result in the same order of regret as that of \DisBE for our modified algorithm. Moreover, since for transmission of each real number $\log(dNT)$ bits is used, the communication cost of our modified algorithm in terms of number of bits is same as that stated in Theorem \ref{thm:regretandCCfpe} with an additional multiplicative factor $\log(dNT)$.

%%%%%%%%%%%%%%%%%%%%%%%%%%%%%%%%%%%%%%%%%%%%%%%%%%%%%%%%%%%%%%%%%%%%%%%%%%%%%%%%%%%%%%%%%%%%%%%%%%%%%%%%%%%%%%%%%%%%%%%%%%%%%%%%%%%%%%%%%%%%%%%%%%%%%%%%%%%%%%%%%%%%%%%%%%%%%%%%%%%%%%%%%%%%%%%%%%%%%%%%%%%%%%%%%%%%%%%%%%%%%%%%%%%%%%%%%%%%%%%%%%%%%%%%%%%%%%%

\subsection{Relaxing the Assumption on Knowledge of $\Dc$}\label{sec:relaxingD}
In this section, we relax this assumption and consider more realistic settings where each agent $i$ can estimate matrix $\Lambda_m^i$ in batch $m$ up to an $\epsilon_m$ error, i.e., 
\begin{align}
    (1-\epsilon_m)\Lambda_m^i\preceq\tilde \Lambda_m^i\preceq (1+\epsilon_m)\Lambda_m^i,\label{eq:twoineq}
\end{align}

where $\tilde\Lambda_m^i$ is an estimation of $\Lambda_m^i$. Given this estimation, we define
\begin{align}
    \tilde\thetab_m^i =\left(\tilde\Lambda_{m}^i\right)^{-1} \sum_{j=1}^N \ub^{j}_{m},
\end{align}
as the new estimation of $\thetab_\ast$ computed by agent $i$ at batch $m$ in this modified version of \DisBE.

We note that if the inequalities hold component-wise, i.e., $(1-\epsilon_m)\Lambda_m^i\leq \tilde \Lambda_m^i\leq (1+\epsilon_m)\Lambda_m^i$, this concludes that \eqref{eq:twoineq} holds. This is because for any positive semi-definite matrices $\A$, $\B$, and $\C$ such that $\A=\B+\C$, we have:
    \begin{align}\label{eq:dett}
    \A\succeq \B,~\A\succeq\C.
    \end{align}
This combined with the fact that all $(1-\epsilon_m)\Lambda_m^i$, $\tilde \Lambda_m^i$, and $(1+\epsilon_m)\Lambda_m^i$ are positive semi-definite symmetric matrices ensures that \eqref{eq:twoineq} holds if $(1-\epsilon_m)\Lambda_m^i\leq \tilde \Lambda_m^i\leq (1+\epsilon_m)\Lambda_m^i$, and therefore, \eqref{eq:twoineq} is a weaker assumption than the component-wise assumption $(1-\epsilon_m)\Lambda_m^i\leq \tilde \Lambda_m^i\leq (1+\epsilon_m)\Lambda_m^i$.

Now, we define corresponding modified confidence intervals in the following lemma.
\begin{lemma}\label{lemm:relaxedconfidencesets}
Suppose $\norm{\boldsymbol\theta_\ast}_2\leq 1$, $\norm{\x^i_{t,a}}_2\leq 1$, $\abs{y_t^i}\leq 1$  for all $(a,i,t)\in[K]\times[N]\times[T]$ and $\epsilon_m\leq \sqrt{\frac{\la}{NT_m}}$ for all $m\in[M]$. For $\delta\in(0,1)$, let 
    $\beta_m=6\sqrt{\frac{\log\left(\frac{2KNT}{\delta}\right)}{1-\epsilon_m}}+4\sqrt{\la}$.
Then for all $\x\in\Xc_t^i,i\in[N],t\in[T], m\in[M]$, with probability at least $1-\delta$, it holds that $\abs{\left\langle\x,\tilde\thetab_m^i-\thetab_\ast\right\rangle}\leq \beta_m\norm{\x}_{\left(\tilde\Lambda_m^i\right)^{-1}}$.
\end{lemma}

\begin{proof}
The proof closely follows the steps in the proof of Lemma \ref{lemm:confidencesets}. For each batch $m\in[M]$, let $\bb_m=\sum_{t=\Tc_{m-1}+1}^{\Tc_{m-1}+T_m/2}\sum_{i=1}^N \x^{i}_{t}y^{i}_{t}$ and $\Vb_m = \sum_{t=\Tc_{m-1}+1}^{\Tc_{m-1}+T_m/2}\sum_{i=1}^N \x^{i}_{t}{\x^{i}_{t}}^\top$. For a fixed $\x\in\Xc_t^i$ and $(i,t)\in[N]\times[T]$, let $z_{t,m}^{j,i}=\x^\top\left(\tilde\Lambda_m^i\right)^{-1}\left(\x_t^jy_t^j-\mathbb{E}_{\Xc\sim\Dc^i_{m}}\mathbb{E}_{\x\sim\pi_{m-1}^i(\Xc)}[\x\x^\top]\thetab_\ast\right)$. Thus, we have
\begin{align}
  \abs{\left\langle\x,\tilde\thetab_m^i-\thetab_\ast\right\rangle}&=\abs{\left\langle\x,\left(\tilde\Lambda_m^i\right)^{-1}\bb_m-\thetab_\ast\right\rangle}\nn\\
    &=\abs{\left\langle\x,\left(\tilde\Lambda_m^i\right)^{-1}\bb_m\right\rangle-\left\langle\x,\left(\tilde\Lambda_m^i\right)^{-1}\tilde\Lambda_m^i\thetab_\ast\right\rangle}\nn\\
    &=\abs{\left\langle\x,\left(\tilde\Lambda_m^i\right)^{-1}\bb_m\right\rangle-\left\langle\x,\left(\tilde\Lambda_m^i\right)^{-1}\left(\Lambda_m^i-\la I\right)\thetab_\ast\right\rangle+\left\langle\x,\left(\tilde\Lambda_m^i\right)^{-1}\left(\Lambda_m^i-\tilde\Lambda_m^i-\la I\right)\thetab_\ast\right\rangle}\nn\\
    &\leq\abs{\left\langle\x,\left(\tilde\Lambda_m^i\right)^{-1}\bb_m\right\rangle-\left\langle\x,\left(\tilde\Lambda_m^i\right)^{-1}\left(\Lambda_m^i-\la I\right)\thetab_\ast\right\rangle}+\abs{\left\langle\x,\left(\tilde\Lambda_m^i\right)^{-1}\left(\Lambda_m^i-\tilde\Lambda_m^i-\la I\right)\thetab_\ast\right\rangle}\nn\\
    &\leq \abs{\x^\top\left(\tilde\Lambda_m^i\right)^{-1}\left(\bb_m-\frac{NT_m}{2}\mathbb{E}_{\Xc\sim\Dc^i_{m}}\mathbb{E}_{\x\sim\pi_{m-1}^i(\Xc)}[\x\x^\top]\thetab_\ast\right)}+4\sqrt{\la}\norm{\x}_{\left(\tilde\Lambda_m^i\right)^{-1}}\tag{Cauchy Schwarz inequality}\\
   &= \abs{\sum_{t=\Tc_{m-1}+1}^{\Tc_{m-1}+T_m/2}\sum_{j=1}^Nz_{t,m}^{j,i}}+ 4\sqrt{\la}\norm{\x}_{\left(\tilde\Lambda_m^i\right)^{-1}}\label{eq:thetermwithlambda},
\end{align}

where the second inequality follows from
\begin{align}
    \norm{\thetab_\ast}_{\left(\tilde\Lambda_m^i\right)^{-1}\left(\Lambda_m^i-\tilde\Lambda_m^i-\la I\right)^2}&=\sqrt{\thetab_\ast^\top\left(\tilde\Lambda_m^i\right)^{-1}\left(\Lambda_m^i-\tilde\Lambda_m^i-\la I\right)^2\thetab_\ast}\nn\\
    &\leq \norm{\thetab_\ast}_2\sqrt{\lamax\left(\left(\tilde\Lambda_m^i\right)^{-1}\left(\Lambda_m^i-\tilde\Lambda_m^i-\la I\right)^2\right)}\nn\\
    &\leq\sqrt{\lamax\left(\left(\tilde\Lambda_m^i\right)^{-1}\left(\Lambda_m^i-\tilde\Lambda_m^i\right)^2+\la^2\left(\tilde\Lambda_m^i\right)^{-1}\right)}\tag{$\norm{\thetab_\ast}_2\leq 1$}\\
    &\leq\sqrt{\lamax\left(\left(\tilde\Lambda_m^i\right)^{-1}\left(\Lambda_m^i-\tilde\Lambda_m^i\right)^2+\la^2\left(\tilde\Lambda_m^i\right)^{-1}\right)}\nn\\
    &\leq\sqrt{\lamax\left(\left(\tilde\Lambda_m^i\right)^{-1}\left(\Lambda_m^i-\tilde\Lambda_m^i\right)^2\right)}+\sqrt{\la}\tag{Cauchy Schwarz inequality}\\
    &\leq \epsilon_m\sqrt{\lamax\left(\tilde\Lambda_m^i\right)}+\sqrt{\la}\tag{Eqn. \eqref{eq:twoineq}}\\
    &\leq 2\epsilon_m\sqrt{\lamax\left(\Lambda_m^i\right)}+\sqrt{\la}\tag{Eqn. \eqref{eq:twoineq}}\\
    &\leq \epsilon_m\sqrt{NT_m}+3\sqrt{\la}\nn\\
    &\leq 4\sqrt{\la}\tag{$\epsilon_m\leq \sqrt{\frac{\la}{NT_m}}$}.
\end{align}

Note that
\begin{align}
    \mathbb{E}\left[z_{t,m}^{j,i}\right] = \mathbb{E}\left[\x^\top\left(\tilde\Lambda_m^i\right)^{-1}\left(\x_t^j({\x_t^j}^\top\thetab_\ast+\eta_t^j)-\mathbb{E}_{\Xc\sim\Dc^i_{m}}\mathbb{E}_{\x\sim\pi_{m-1}^i(\Xc)}[\x\x^\top]\thetab_\ast\right)\right] = 0 \tag{Noise $\eta_t^j$ is zero-mean and independent of $\x_t^j$},
\end{align}
By
Azuma’s inequality, for a fixed $\x\in\Xc_t^i$ and $(i,t)\in[N]\times[T]$, we have 
\begin{align}
    \mathbb{P}\left(\abs{\sum_{t=\Tc_{m-1}+1}^{\Tc_{m-1}+T_m/2}\sum_{j=1}^Nz_{t,m}^{j,i}}\geq \alpha\norm{\x}_{\left(\tilde\Lambda_m^i\right)^{-1}}\right)&\leq 2{\rm exp}\left(\frac{-\alpha^2\norm{\x}_{\left(\tilde\Lambda_m^i\right)^{-1}}^2}{2c_m^i}\right),\label{eq:azuma1}
\end{align}
where 
\begin{align}
    c_m^i &= \sum_{t=\Tc_{m-1}+1}^{\Tc_{m-1}+T_m/2}\sum_{j=1}^N\abs{\x^\top\left(\tilde\Lambda_m^i\right)^{-1}\left(\x_t^jy_t^j-\mathbb{E}_{\Xc\sim\Dc^i_{m}}\mathbb{E}_{\x\sim\pi_{m-1}^i(\Xc)}[\x\x^\top]\thetab_\ast\right)}^2\nn\\
    &\leq 2\sum_{t=\Tc_{m-1}+1}^{\Tc_{m-1}+T_m/2}\sum_{j=1}^N\abs{\x^\top\left(\tilde\Lambda_m^i\right)^{-1}\x_t^jy_t^j}^2+NT_m\abs{\x^\top\left(\tilde\Lambda_m^i\right)^{-1}\mathbb{E}_{\Xc\sim\Dc^i_{m}}\mathbb{E}_{\x\sim\pi_{m-1}^i(\Xc)}[\x\x^\top]\thetab_\ast}^2\nn\\
    &\leq 2\sum_{t=\Tc_{m-1}+1}^{\Tc_{m-1}+T_m/2}\sum_{j=1}^N\abs{\x^\top\left(\tilde\Lambda_m^i\right)^{-1}\x_t^j}^2+\frac{4}{NT_m}\abs{\x^\top\left(\tilde\Lambda_m^i\right)^{-1}\left(\Lambda_m^i-\la I\right)\thetab_\ast}^2\nn\\
    &= 2\sum_{t=\Tc_{m-1}+1}^{\Tc_{m-1}+T_m/2}\sum_{j=1}^N\x^\top\left(\tilde\Lambda_m^i\right)^{-1}\x_t^j{\x_t^j}^\top\left(\tilde\Lambda_m^i\right)^{-1}\x+\frac{4}{NT_m}\abs{\x^\top\left(\tilde\Lambda_m^i\right)^{-1}\left(\Lambda_m^i\right)\thetab_\ast-\la\x^\top\left(\tilde\Lambda_m^i\right)^{-1}\thetab_\ast}^2\nn\\
    &\leq 2\x^\top\left(\tilde\Lambda_m^i\right)^{-1}\Vb_m\left(\tilde\Lambda_m^i\right)^{-1}\x+\frac{1}{1-\epsilon_m}\left(4+\frac{8\la}{NT_m}\right)\norm{\x}_{\left(\tilde\Lambda_m^i\right)^{-1}}^2\tag{Cauchy Schwarz inequality}\\
    &\leq 4\x^\top\left(\tilde\Lambda_m^i\right)^{-1}\Lambda_m^i\left(\tilde\Lambda_m^i\right)^{-1}\x+\frac{1}{1-\epsilon_m}\left(4+\frac{8\la}{NT_m}\right)\norm{\x}_{\left(\tilde\Lambda_m^i\right)^{-1}}^2\tag{Conditioned on the event in Eqn. \eqref{eq:lowerboundonLambdam}}\\
    &\leq\frac{4}{1-\epsilon_m}\x^\top\left(\tilde\Lambda_m^i\right)^{-1}\x+\frac{1}{1-\epsilon_m}\left(4+\frac{8\la}{NT_m}\right)\norm{\x}_{\left(\tilde\Lambda_m^i\right)^{-1}}^2\tag{$(1-\epsilon_m)\Lambda_m^i\preceq\tilde \Lambda_m^i$}\\
    &=\frac{8}{1-\epsilon_m}\left(1+\frac{\la}{NT_m}\right)\norm{\x}_{\left(\tilde\Lambda_m^i\right)^{-1}}^2\nn\\
    &\leq \frac{16}{1-\epsilon_m}\norm{\x}_{\left(\tilde\Lambda_m^i\right)^{-1}}^2,\label{eq:upperboundoncm1}
\end{align}
where the third inequity follows from the fact that 

% \begin{align}
%     \norm{\x}_{\left(\Lambda_m^i\right)^{-1}}^2\geq \frac{\norm{\x}^2}{\lamax\left(\Lambda_m^i\right)}\geq \frac{2H^2}{NT_m}.
% \end{align}

\begin{align}
\thetab_\ast^\top \left(\Lambda_m^i\left(\tilde\Lambda_m^i\right)^{-1}\Lambda_m^i\right) \thetab_\ast &\leq\norm{\thetab_\ast}^2 \lamax\left(\Lambda_m^i\left(\tilde\Lambda_m^i\right)^{-1}\Lambda_m^i\right)\nn\\
&\leq \lamax\left(\Lambda_m^i\left(\tilde\Lambda_m^i\right)^{-1}\Lambda_m^i\right) \tag{$\norm{\thetab_\ast}_2\leq 1$}\\
&\leq \frac{1}{1-\epsilon_m}\lamax\left(\Lambda_m^i\right)\tag{$(1-\epsilon_m)\Lambda_m^i\preceq\tilde \Lambda_m^i$}\\
&\leq\frac{\la+NT_m}{1-\epsilon_m}.\nn
\end{align}
Combining \eqref{eq:thetermwithlambda}, \eqref{eq:azuma1} and \eqref{eq:upperboundoncm1}, and by a union bound, we have
\begin{align}
    \mathbb{P}\left(\abs{\left\langle\x,\thetab_m^i-\thetab_\ast\right\rangle}\leq  \left(6\sqrt{\frac{\log\left(\frac{2KNT}{\delta}\right)}{1-\epsilon_m}}+4\sqrt{\la}\right)\norm{\x}_{\left(\tilde\Lambda_m^i\right)^{-1}},~\forall \x\in\Xc_t^i,i\in[N],t\in[T], m\in[M]\right)\geq 1-\delta.
\end{align}
\end{proof}

Now, we state the regret bound for \DisBE with $\tilde\Lambda_m^i$ and $\tilde\thetab_m^i$.

\begin{theorem}\label{thm:regretandCCfperelaxed}
Fix $M = 1+\log\left(\log\left(NT/d\right)/2+1\right)$. Under the setting of Lemma \ref{lemm:relaxedconfidencesets}, if $T\geq \Omega\left(d^{22}\log^2(NT/\delta)\log^2 d\log^2(d\la^{-1})\right)$ and $\beta = \max_{m\in[M]}\beta_m$, then with probability at least $1-2\delta$, it holds that
$R_T\leq\Oc\left(\frac{1}{1-\max_{m\in[M]}\epsilon_m}\sqrt{dNT\log d \log^2\left(\frac{KNT}{\delta\la}\right)}\log\log\left(\frac{NT}{d}\right)\right)$, where the communication cost is measured by the number of real numbers communicated by the agents.
\end{theorem}

\begin{proof}
The proof follows similar steps to those in the proof of Theorem \ref{thm:regretandCCfpe}.

We focus on the regret of the $i$-th agent at $m$-th batch for any $m\geq 3$. Let $\Dc_m^i$ be the distribution based on which the surviving sets $\Xc_t^{i(m)}$ for all $t\in[\Tc_{m-1}+1:\Tc_m]$ are generated when conditioned on the first $m-1$ batches. For any $t\in[\Tc_{m-1}+1:\Tc_m]$, conditioned on the event that the confidence intervals in Lemma \ref{lemm:confidencesets} hold, we have
\begin{align}
    r^{i}_{t} &=\mathbb{E}\left[ \langle\thetab_\ast,\x^{i}_{\ast,t}\rangle-\langle\thetab_\ast,\x^{i}_{t}\rangle\right] \nn\\
    &\leq \mathbb{E}\left[ \langle\tilde\thetab_{m-1}^i,\x^{i}_{\ast,t}\rangle-\langle\tilde\thetab_{m-1}^i,\x^{i}_{t}\rangle+\beta\norm{\x^{i}_{\ast,t}}_{\left(\tilde\Lambda_{m-1}
    ^i\right)^{-1}}+\beta\norm{\x^{i}_{t}}_{\left(\tilde\Lambda_{m-1}
    ^i\right)^{-1}}\right] \tag{Lemma \ref{lemm:relaxedconfidencesets}}\\
     &\leq 2\beta \mathbb{E}\left[\norm{\x^{i}_{\ast,t}}_{\left(\tilde\Lambda_{m-1}
    ^i\right)^{-1}}+\norm{\x^{i}_{t}}_{\left(\tilde\Lambda_{m-1}
    ^i\right)^{-1}}\right] \tag{$\x^{i}_{\ast,t}\in\Xc_t^{i(m)}$}\\
     &\leq 4\beta \mathbb{E}\left[\max_{\x\in\Xc_t^{i(m)}}\norm{\x}_{\left(\tilde\Lambda_{m-1}
    ^i\right)^{-1}}\right]\nn\\
     &\leq 4\beta \mathbb{E}_{\Xc\sim\Dc_m^i}\left[\max_{\x\in\Xc}\norm{\x}_{\left(\tilde\Lambda_{m-1}
    ^i\right)^{-1}}\right]\nn\\
     &\leq 4\beta \mathbb{E}_{\Xc\sim\Dc^i_{m-1}}\left[\max_{\x\in\Xc}\norm{\x}_{\left(\tilde\Lambda_{m-1}
    ^i\right)^{-1}}\right]\nn\\
    &\leq \frac{4\beta}{\sqrt{1-\epsilon_m}} \mathbb{E}_{\Xc\sim\Dc^i_{m-1}}\left[\max_{\x\in\Xc}\norm{\x}_{\left(\Lambda_{m-1}
    ^i\right)^{-1}}\right]\tag{$(1-\epsilon_m)\Lambda_m^i\preceq\tilde \Lambda_m^i$}\\
     &\leq \frac{8\beta}{\sqrt{NT_{m-1}(1-\epsilon_m)}} \mathbb{E}_{\Xc\sim\Dc^i_{m-1}}\left[\max_{\x\in\Xc}\sqrt{\x^\top \left(\frac{2\la}{NT_{m-1}} I+\mathbb{E}_{\Xc\sim\Dc^i_{m-1}} \mathbb{E}_{\y\sim\pi_{m-2}^i(\Xc)}[\y\y^\top]\right)^{-1}\x}\right]\nn\\
     &=\frac{8\beta}{\sqrt{NT_{m-1}(1-\epsilon_m)}}{{\mathbb{V}}}^{(\frac{2\la}{NT_{m-1}})}_{\Dc^i_{m-1}}(\pi_{m-2}^i)
     \label{eq:tobecontinuedfperelaxed},
\end{align}
where the third inequality follows from our established confidence intervals in Lemma \ref{lemm:relaxedconfidencesets} guaranteeing that $\x_{\ast,t}^i\in\Xc_t^{i(m)}$ for all $(i,t,m)\in[N]\times[\Tc_{m-1}+1:\Tc_m]\times[M]$ with probability at least $1-
\delta$. The rest of the proof follows the steps as those in the proof of Theorem \ref{thm:regretandCCfpe} with an additional $\frac{1}{\sqrt{1-\epsilon_m}}$ multiplicative factor in the bound.

Therefore, we conclude that, with probability at least $1-2\delta$, it holds that
\begin{align}
    R_T &\leq 4\sqrt{dNT}\left(\frac{NT}{d}\right)^{\frac{1}{2(2^{M-1}-1)}}+8 \beta M\sqrt{\frac{dNT\log d\log (NT\la^{-1})}{1-\max_{m\in[M]}\epsilon_m}} \left(\frac{NT}{d}\right)^{\frac{1}{2(2^{M-1}-1)}}\nn\\
    &\leq \Oc\left(\frac{1}{1-\max_{m\in[M]}\epsilon_m}\sqrt{dNT\log d \log^2\left(\frac{KNT}{\delta\la}\right)}\log\log\left(\frac{NT}{d}\right)\right).
\end{align}

\end{proof}

\section{DECENTRALIZED BATCH ELIMINATION LUCB WITHOUT SERVER}\label{sec:DPEapp}

In this environment, the agents are represented by the nodes of an
undirected and connected graph $G$. Each agent $i$ can send and receive messages only to and from its immediate neighbors $j\in\Nc(i)$.

\begin{definition}[Communication Matrix]\label{def:comm_matrix}
For an undirected connected graph $G$ with $N$ nodes, $\mathbf{P}\in \mathbb{R}^{N\times N}$ is a symmetric communication matrix if it satisfies the following three conditions: (i) $\mathbf{P}_{i,j}=0$ if there is no connection between nodes $i$ and $j$; (ii) the sum of each row and column of $\mathbf{P}$ is 1; (iii) the eigenvalues are real and their magnitude is less than 1, 
%absolute values of all its eigenvalues are real and less than 1, 
i.e., $1=|\la_1|>|\la_2|\geq\ldots|\la_N|\geq0$.
\end{definition}
% It is well-known that a matrix $\mathbf{P}$ satisfies the conditions above if and only if $\lim_{t\to\infty} \mathbf{P}^t = \mathbf{1}\mathbf{1}^T/N$, where $\mathbf{1}$ is the vector of ones. 

%Our algorithms require knowledge of a communication matrix $\mathbf{P}$ corresponding to the underlying graph. 
We assume that $\mathbf{P}$ is known to the agents. We remark that $\mathbf{P}$ can be constructed with little global information about the graph, such as its adjacency matrix and the graph's maximal degree; For example, one can compute it as $\mathbf{P} = I_N-\frac{1}{\delta_{{\rm max}}+1}\D^{-1/2}\mathcal{L}\D^{-1/2}$, where $\delta_{{\rm max}}$ is the maximum degree of the graph, $\mathcal{L}\in \mathbb{R}^{N\times N}$ is the graph Laplacian, and $\D\in \mathbb{R}^{N\times N}$ is a diagonal matrix whose entries are the degrees of the nodes (see \cite{duchi2011dual} for details).

\paragraph{Running consensus.}
In order to share information about agents' past actions among the network, we rely on \emph{running consensus}, e.g., \cite{lynch1996distributed,xiao2004fast}. The goal of running consensus is that after enough rounds of communication, each agent has an accurate estimate of the average (over all agents) of the initial values of each agent. Precisely, let $\boldsymbol{\nu}_0\in\mathbb{R}^N$ be a vector, where each entry $\boldsymbol{\nu}_{0,i}, i\in[N]$ represents agent's $i$ information at some initial round. Then, running consensus aims at providing an accurate estimate of the average $\frac{1}{N}\sum_{i\in[N]}\boldsymbol{\nu}_{0,i}$ for each agent.
% Note that encoding $\boldsymbol{\nu}_0=X\mathbf{e}_j$, allows all agents to eventually get an estimate of the value $X=\sum_{i\in[N]}\boldsymbol{\nu}_{0,i}$ that was initially known only to agent $j$. 
It turns out that the communication matrix $\mathbf{P}$ defined in Definition \ref{def:comm_matrix} plays a key role in reaching consensus.
The details are standard in the rich related literature \cite{xiao2004fast,lynch1996distributed}. Here, we only give a brief explanation of the high-level principles. Roughly speaking, a consensus algorithm updates $\boldsymbol{\nu}_{0}$ by $\boldsymbol{\nu}_1=\mathbf{P}\boldsymbol{\nu}_{0}$, $\boldsymbol{\nu}_2=\mathbf{P}\boldsymbol{\nu}_{1}$ and so on.
Note that this operation respects the network structure since the updated value $\boldsymbol{\nu}_{1,j}$ is a weighted average of only $\boldsymbol{\nu}_{0,j}$ itself and neighbor-only values $\boldsymbol{\nu}_{0,i},i\in\Nc(j).$ Thus, after $S$ rounds, agent $j$ has access to entry $j$ of  $\boldsymbol{\nu}_{S}=\mathbf{P}^S\boldsymbol{\nu}_{0}$. We adapt \emph{polynomial filtering} introduced in \cite{martinez2019decentralized,seaman2017optimal} to speed up the mixing of information by following an approach whose convergence rate is faster than the standard multiplication method above. Specifically, after $S$ communication rounds, instead of $\mathbf{P}^S$, agents compute  and apply to the initial vector $\boldsymbol{\nu}_0$ an appropriate re-scaled \emph{Chebyshev polynomial} $q_S(\mathbf{P})$ of degree $S$ of the communication matrix.  Recall that Chebyshev polynomials %\cite{young2014iterative}
are defined recursively. It turns out that the Chebyshev polynomial of degree $\ell$ for a communication matrix $\mathbf{P}$ is also given by a recursive formula as follows: 
$
    q_{\ell+1}(\mathbf{P})= \frac{2w_\ell}{|\la_2|w_{\ell+1}}\mathbf{P}q_{\ell}(\mathbf{P})-\frac{w_{\ell-1}}{w_{\ell+1}}q_{\ell-1}(\mathbf{P}),
$
where $w_0 = 0, w_1 = 1/|\la_2|$, $w_{\ell+1}=2w_\ell/|\la_2|-w_{\ell-1}$, $q_0(\mathbf{P})=I$ and $q_1(\mathbf{P})=\mathbf{P}$.  Specifically, in a Chebyshev-accelerated gossip protocol \cite{martinez2019decentralized}, the agents update their estimates of the average of the initial vector's $\boldsymbol{\nu}_0$ entries as follows:
\begin{align}
\!\!   \boldsymbol \nu_{\ell+1} &= {(2w_\ell)}/{(|\la_2|w_{\ell+1})}\mathbf{P}\boldsymbol \nu_{\ell}-{(w_{\ell-1}}/{w_{\ell+1})}\boldsymbol \nu_{\ell-1}.\label{eq:recursion}
\end{align}
\DecBE, presented in Algorithm \ref{alg:DPE}, implements the Chebyshev-accelerated gossip protocol outlined above for every entry of vectors $\ub^{i}_{m}= \sum_{t=\Tc_{m-1}+1}^{\Tc_{m-1}+T_m/2} \x^{i}_{t}y^{i}_{t}$ at the end of $m$-th batch.

% Specifically, we summarize the accelerated communication step described in \eqref{eq:recursion} with a function ${\rm Comm}(x_{\rm now},x_{\rm prev},\ell)$ with three inputs: (1) $x_{\rm now}$, the quantity of interest that the agent wants to update at the current round; (2) $x_{\rm prev}$, the estimated value for the same quantity of interest that the agent updated in the previous round (cf. $\boldsymbol{\nu}_{\ell-1}$ in \eqref{eq:recursion}); (3) $\ell$, the current communication round. Note that inputs here are scalars, however, matrices and vectors can also be passed as inputs, in which case $\rm Comm$ runs entrywise. For a detailed description of $\rm Comm$ please refer to Algorithm \ref{alg:comm} in Appendix \ref{sec:comm}.

The accelerated consensus algorithm, summarized in Algorithm \ref{alg:comm}, guarantees fast mixing of information thanks to the following key property stated in Lemma 3 of \cite{martinez2019decentralized}: for $\epsilon\in(0,1)$ and any vector $\boldsymbol{\nu_0}$ in the $N$-dimensional simplex, it holds that 
\begin{equation}\label{eq:cheb}  \|Nq_S(\mathbf{P})\boldsymbol{\nu_0} - {\mathbf{1}}\|_2\leq {\epsilon},~\text{if}~S= \frac{\log(2N/\epsilon)}{\sqrt{2\log(1/|\la_2|)}}.
\end{equation}

In view of this, \DecBE properly implements the accelerated consensus algorithm such that for every $i\in[N]$ and $m\in[M]$, the vector $\ub_{m}^i$ is communicated within the network during the last $S$ rounds of batch $m$. At round $\Tc_m+1$, agent $i$ has access to $\sum_{j=1}^Na_{i,j}\ub_{m}^j$, where $a_{i,j}=N[q_S(\mathbf{P})]_{i,j}$. Thanks to \eqref{eq:cheb}, $a_{i,j}$ is $\epsilon$ close to $1$, thus, these are good approximations of the true $\sum_{j=1}^N\ub_{m}^j$. Furthermore, the choice of grid $\Tc = \{\Tc_0,\Tc_1,\ldots,\Tc_M\}$ in \DecBE is slightly different than what used in \DisBE.

%%%%%%%%%%%%%%%%%%%%%%%%%%%%%%%%%%%%%%%%%%%%%%%%%%%%%%%%%%%%%%%%%%%%%%%%%%%%%%%%%%%%%%%%%%%%%%%%%%%%%%%%%%%%%%%%%%%%%%%%%%%%%%%%%%%%%%%%%%%%%%%%%%%%%%%%%%%%%%%%%%%%%%%%%%%%%%%%%%%%%%%%%%%%%%%%%%%%%%%%%%%%%%%%%%%%%%%%%%%%%%%%%%%%%%%%%%%%%%%%%%%%%%%%%%%%%%%

\subsection{Theoretical guarantees of \DecBE}

\begin{algorithm}[t!]
   \caption{\DecBE for agent $i$}
   \label{alg:DPE}
\DontPrintSemicolon
\KwInput{$N$, $d$, $\delta$, $T$, $M$, $\lambda$, $\epsilon$}
 {\bf Initialization:} $S= \frac{\log(2N/\epsilon)}{\sqrt{2\log(1/|\la_2|)}}$, $a = \sqrt{T+S}\left(\frac{N(T+S)}{d}\right)^{\frac{1}{2(2^{M-1}-1)}}$,  $T_1=T_2=a\sqrt{\frac{d}{N}}+S$, $T_m=\lfloor a\sqrt{T_{m-1}-S}+S\rfloor$, $\thetab_{0}^i=\mathbf{0}$, $\Lambda_0^i=\lambda I$, $\Tc_0=0$, $\Tc_m = \Tc_{m-1}+T_m$, $\la=5\log\left(\frac{4dT}{\delta}\right)$, $\gamma=12\sqrt{\log\left(\frac{2KNT}{\delta}\right)}+2\sqrt{\la}$, arbitrary policy $\pi_0^i$\;
  \For{$m=1,\ldots,M$}{
\For{$t=\Tc_{m-1}+1,\ldots,\Tc_m-S$}{
   Let $\Xc_t^{i(m)}=\cap_{k=0}^{m-1}\Ec\left(\Xc_t^i;(\Lambda_k^i,\hat\thetab_k^i,\gamma)\right)$\;
Play arm $a_{i,t}$ associated with feature vector $\x^{i}_{t}\sim \pi_{m-1}\left(\Xc_t^{i(m)}\right)$ and observe $y^{i}_{t}$.
}
Set $\Kc_0^i = \sum_{t=\Tc_{m-1}+1}^{\Tc_{m-1}+(T_m-S)/2} \x^{i}_{t}y^{i}_{t}$\;
\For{$t=\Tc_{m}-S+1$}{
Let $\Xc_t^{i(m)}=\cap_{k=0}^{m-1}\Ec\left(\Xc_t^i;(\Lambda_k^i,\hat\thetab_k^i,\gamma)\right)$\;
Play arm $a_{i,t}$ associated with feature vector $\x^{i}_{t}\sim \pi_{m-1}\left(\Xc_t^{i(m)}\right)$ and observe $y^{i}_{t}$.\;
Send each entry of $\Kc_{0}^i$, i.e.,  $[\Kc_{0}^i]_{n},~\forall n\in[d]$ to your neighbors $\Nc(j)$ and receive the corresponding values from them. For each $n\in[d]$, update $[\Kc_{1}^i]_n = \mathbf{P}_{i,i} [\Kc_{0}^i]_n+\sum_{j\in \Nc(i)}\mathbf{P}_{i,j}[\Kc_{0}^j]_n$
}
Set $s=1$\;
\For{$t=\Tc_{m}-S+2,\ldots,\Tc_m$}{
Construct set $\Xc_t^{i(m)}=\cap_{k=0}^{m-1}\Ec\left(\Xc_t^i;(\Lambda_k^i,\hat\thetab_k^i,\gamma)\right)$.\;
Play arm $a_{i,t}$ associated with feature vector $\x^{i}_{t}\sim \pi_{m-1}\left(\Xc_t^{i(m)}\right)$ and observe $y^{i}_{t}$.\;
$[\Kc_{s+1}^i]_{n} = \Comm([\Kc_{s}^i]_{n},[\Kc_{s-1}^i]_{n},s+1)$, $\forall n\in[d]$\;
$s=s+1$
}
 Compute/construct 
   \begin{align}
       \Lambda_{m}^i&=  \lambda I+\frac{N(T_m-S)}{2}\mathbb{E}_{\Xc\sim\Dc^i_{m}}\mathbb{E}_{\x\sim\pi^i_{m-1}(\Xc)}[\x\x^\top],\nn\\
       \hat\thetab_{m}^i &= \left(\Lambda_{m}^i\right)^{-1} \bar \ub_{m,i},\nn\\
       \Sc_m^i&=\left\{\Xc_t^{i(m+1)}\right\}_{t=\Tc_{m-1}+(T_m-S)/2+1}^{\Tc_m},\nn\\
       \pi_m^i&= \Exp\left(\frac{2\la}{N(T_m-S)},\Sc_m^i\right).\nn
   \end{align}
}
\end{algorithm}

As the first step in regret analysis of \DecBE, we establish the following confidence intervals.

\begin{lemma}[Confidence intervals for \DecBE]\label{lemm:confidencesets1}
Suppose Assumption \ref{assum:boundedness} holds. Fix $\delta\in(0,1)$ and let $\epsilon = \frac{\beta}{\sqrt{d}}$ and $\gamma=2\beta$, where $\beta$ is defined in Lemma \ref{lemm:confidencesets}. Then
\begin{align}
    \mathbb{P}\left(\abs{\left\langle\x,\hat\thetab_m^i-\thetab_\ast\right\rangle}\leq \gamma\norm{\x}_{\left(\Lambda_m^i\right)^{-1}},~\forall \x\in\Xc_t^i,i\in[N],t\in[T], m\in[M]\right)\geq 1-\delta.
\end{align}
\end{lemma}
\begin{proof}
Recall the definition of $\thetab_m^i$ in \eqref{eq:lse}. For a fixed $\x\in\Xc_t^i$ and $(i,t)\in[N]\times[T]$, we have
\begin{align}
    \abs{\left\langle\x,\hat\thetab_m^i-\thetab_\ast\right\rangle} &\leq \abs{\left\langle\x,\thetab_m^i-\thetab_\ast\right\rangle}+\abs{\left\langle\x,\hat\thetab_m^i-\thetab_m^i\right\rangle}\nn\\
    &\leq \abs{\left\langle\x,\thetab_m^i-\thetab_\ast\right\rangle}+\norm{\x}_{\left(\Lambda_m^i\right)^{-2}}\norm{\bar\ub_{m,i}-\sum_{j=1}^N \ub^{j}_{m}}_2\tag{Cauchy Schwarz inequality}\\
    &\leq \abs{\left\langle\x,\thetab_m^i-\thetab_\ast\right\rangle}+\epsilon\sqrt{d}\norm{\x}_{\left(\Lambda_m^i\right)^{-1}}\tag{Assumption \ref{assum:boundedness} and choice of $S$ in \eqref{eq:cheb}}\nn\\
    &= \abs{\left\langle\x,\thetab_m^i-\thetab_\ast\right\rangle}+\beta\norm{\x}_{\left(\Lambda_m^i\right)^{-1}}\label{eq:proofofsecondconfidenceset}.
\end{align}
Combining Lemma \ref{lemm:confidencesets} and \eqref{eq:proofofsecondconfidenceset}, we have
\begin{align}
    \mathbb{P}\left(\abs{\left\langle\x,\hat\thetab_m^i-\thetab_\ast\right\rangle}\leq 2\beta\norm{\x}_{\left(\Lambda_m^i\right)^{-1}},~\forall \x\in\Xc_t^i,i\in[N],t\in[T], m\in[M]\right)\geq 1-\delta.
\end{align}
\end{proof}

\begin{theorem}\label{thm:regretandCCdpe}
Fix $M = 1+\log\left(\frac{\log\left(\frac{N(T+S)}{d}\right)}{2}+1\right)$, with $S$ defined in \eqref{eq:cheb} for $\epsilon = 6\sqrt{\frac{\log\left(\frac{2dKNT}{\delta}\right)}{d}}$ in Algorithm \ref{alg:FPE}. Suppose Assumption \ref{assum:boundedness} holds. If $T\geq \Omega\left(d^{22}\log^2(\frac{NT}{\delta})\log^2 d\log^2(d\la^{-1})\right)$, then with probability at least $1-2\delta$, it holds that
\begin{align}\label{eq:regretdpe}
    R_T\leq   \Oc\left(\left(\frac{N\log(dN)}{\sqrt{1/\abs{\la_2}}}+\sqrt{dN\left(T+\frac{\log(dN)}{\sqrt{1/\abs{\la_2}}}\right)\log d \log^2\left(\frac{KN\left(T+\frac{\log(dN)}{\sqrt{1/\abs{\la_2}}}\right)}{\delta\la}\right)}\right)\log\log\left(\frac{NT}{d}\right)\right),
\end{align}
and 
\begin{align}\label{eq:commcostdpe}
    \text{Communication Cost}\leq  \Oc\left( \frac{\delta_{\rm max}dN\log(dN)}{\sqrt{\log(1/\abs{\lambda_2})}}\right).
\end{align}
\end{theorem}
\begin{proof}
The proof follows similar steps as those of Theorem \ref{thm:regretandCCfpe}'s proof. We focus on the regret of $m$-th batch for any $m\geq 3$. For any $i\in[N]$, $t\in[\Tc_{m-1}+1:\Tc_m]$, conditioned on the event that the confidence intervals in Lemma \ref{lemm:confidencesets1} hold, we have
\begin{align}
    r^{i}_{t} &=\mathbb{E}\left[ \langle\thetab_\ast,\x^{i}_{\ast,t}\rangle-\langle\thetab_\ast,\x^{i}_{t}\rangle\right] \nn\\
    &\leq \mathbb{E}\left[ \langle\hat\thetab_{m-1}^i,\x^{i}_{\ast,t}\rangle-\langle\hat\thetab_{m-1}^i,\x^{i}_{t}\rangle+\beta\norm{\x^{i}_{\ast,t}}_{\left(\Lambda_{m-1}
    ^i\right)^{-1}}+\beta\norm{\x^{i}_{t}}_{\left(\Lambda_{m-1}
    ^i\right)^{-1}}\right] \tag{Lemma \ref{lemm:confidencesets1}}\\
     &\leq 2\gamma \mathbb{E}\left[\norm{\x^{i}_{\ast,t}}_{\left(\Lambda_{m-1}
    ^i\right)^{-1}}+\norm{\x^{i}_{t}}_{\left(\Lambda_{m-1}
    ^i\right)^{-1}}\right] \tag{$\x^{i}_{\ast,t}\in\Xc_t^{i(m)}$}\\
     &\leq 4\gamma \mathbb{E}\left[\max_{\x\in\Xc_t^{i(m)}}\norm{\x}_{\left(\Lambda_{m-1}
    ^i\right)^{-1}}\right]\nn\\
     &\leq 4\gamma \mathbb{E}_{\Xc\sim\Dc_m^i}\left[\max_{\x\in\Xc}\norm{\x}_{\left(\Lambda_{m-1}
    ^i\right)^{-1}}\right]\nn\\
     &\leq 4\gamma \mathbb{E}_{\Xc\sim\Dc^i_{m-1}}\left[\max_{\x\in\Xc}\norm{\x}_{\left(\Lambda_{m-1}
    ^i\right)^{-1}}\right]\nn\\
     &\leq \frac{8\gamma}{\sqrt{N(T_{m-1}-S)}} \mathbb{E}_{\Xc\sim\Dc^i_{m-1}}\left[\max_{\x\in\Xc}\sqrt{\x^\top \left(\frac{2\la}{N(T_{m-1}-S)} I+\mathbb{E}_{\Xc\sim\Dc^i_{m-1}}\mathbb{E}_{\y\sim\pi^i_{m-2}(\Xc)}[\y\y^\top]\right)^{-1}\x}\right]\nn\\
     &=\frac{8\gamma}{\sqrt{N(T_{m-1}-S)}}{{\mathbb{V}}}^{(\frac{2\la}{N(T_{m-1}-S)})}_{\Dc^i_{m-1}}(\pi_{m-2}^i)
     \label{eq:tobecontinueddpe},
\end{align}
where the third inequality follows from our established confidence intervals in Lemma \ref{lemm:confidencesets1} guaranteeing that $\x_{\ast,t}^i\in\Xc_t^{i(m)}$ for all $(i,t,m)\in[N]\times[\Tc_{m-1}+1:\Tc_m]\times[M]$ with probability at least $1-
\delta$. Now, continuing form \eqref{eq:tobecontinuedfpe}, we bound the cumulative regret of batches $m\geq 3$, as follows:
\begin{align}
    \sum_{t=\Tc_2+1}^T\sum_{i=1}^N r^{i}_{t} &\leq 2MSN+\sum_{m=3}^M\sum_{t=\Tc_{m-1}+1}^{\Tc_m-S}\sum_{i=1}^N  r_t^i\nn\\
    &\leq 2MSN+\frac{8\gamma M N(T_{m}-S)}{\sqrt{N(T_{m-1}-S)}}{{\mathbb{V}}}^{(\frac{2\la}{N(T_{m-1}-S)})}_{\Dc^i_{m-1}}(\pi_{m-2}^i)\nn\\
    &\leq 2MSN++8 \gamma M \sqrt{dN\log d\log (NT\la^{-1})} \sum_{m=2}^M \frac{T_m-S}{\sqrt{T_{m-1}-S}}\tag{Conditioned on the event in Eqn. \eqref{eq:upperboundonvlambdapi}}\\
 &= 2MSN+8 \gamma Ma\sqrt{dN\log d\log (NT\la^{-1})}.\label{eq:regretfpeofbatchesmgreaterthan2dpe}
\end{align}

Next, we bound cumulative regret of the first two batches. Under Assumption \ref{assum:boundedness}, during the first two batches, the instantaneous regret of each agent $i$ at any round $t$ is at most 2. Therefore 
\begin{align}
    \sum_{t=1}^{\Tc_2}\sum_{i=1}^N r^{i}_{t} \leq 2N\Tc_2 = 4a\sqrt{dN}.\label{eq:regretoffirstepochdpe}
\end{align}
Note that the choice of $a$ in the algorithm ensures that for any $M>0$, $T_M = T$ and $\sum_{m=1}^MT_m\geq T_M=T$, and thus the choice of grid $\{\Tc_1,\ldots,\Tc_M\}$ is valid. If we let $M = 1+\log\left(\frac{\log\left(\frac{N(T+S)}{d}\right)}{2}+1\right)$, from \eqref{eq:regretfpeofbatchesmgreaterthan2dpe} and \eqref{eq:regretoffirstepochdpe}, we conclude that, with probability at least $1-2\delta$, it holds that
\begin{align}
    R_T &\leq 2MSN+ 4\sqrt{dN(T+S)}\left(\frac{NT}{d}\right)^{\frac{1}{2(2^{M-1}-1)}}+8 \gamma M\sqrt{dNT\log d\log (NT\la^{-1})} \left(\frac{N(T+S)}{d}\right)^{\frac{1}{2(2^{M-1}-1)}}\nn\\
    &\leq \Oc\left(\left(\frac{N\log(dN)}{\sqrt{1/\abs{\la_2}}}+\sqrt{dN\left(T+\frac{\log(dN)}{\sqrt{1/\abs{\la_2}}}\right)\log d \log^2\left(\frac{KN\left(T+\frac{\log(dN)}{\sqrt{1/\abs{\la_2}}}\right)}{\delta\la}\right)}\right)\log\log\left(\frac{NT}{d}\right)\right).
\end{align}
\end{proof}

%%%%%%%%%%%%%%%%%%%%%%%%%%%%%%%%%%%%%%%%%%%%%%%%%%%%%%%%%%%%%%%%%%%%%%%%%%%%%%%%%%%%%%%%%%%%%%%%%%%%%%%%%%%%%%%%%%%%%%%%%%%%%%%%%%%%%%%%%%%%%%%%%%%%%%%%%%%%%%%%%%%%%%%%%%%%%%%%%%%%%%%%%%%%%%%%%%%%%%%%%%%%%%%%%%%%%%%%%%%%%%%%%%%%%%%%%%%%%%%%%%%%%%%%%%%%%%%

\subsection{Communication Step}\label{sec:comm}
In this section, we summarize the accelerated Chebyshev communication step, discussed above, in Algorithm \ref{alg:comm}, which follows the same steps as those of the communication algorithm presented in \cite{martinez2019decentralized}.   
\begin{algorithm}[ht]
   \caption{Comm for Agent $i$}
   \label{alg:comm}
 \DontPrintSemicolon
\KwInput{$x_{\rm now}$, $x_{\rm prev}$, $\ell$\;
   {\bf Output:} $x_{i,{\rm next}}$}
   {\bf Initialization:} $w_0 = 0, w_1 = 1/|\la_2|, w_{r} = 2w_{r-1}/|\la_2|-w_{r-2},~\forall 2\leq r \leq S$, $x_{i,{\rm now}} = x_{\rm now}$, $x_{i,{\rm prev}} = x_{\rm prev}$\;
Send $x_{i,{\rm now}}$ and receive the corresponding $x_{j,{\rm now}}$ to and from $j\in\mathcal{N}(i)$ \quad//~Recall that all agents run Comm in parallel.\;
$x_{i,{\rm next}} = \frac{2w_{\ell-1}}{|\la_2|w_{\ell}}\mathbf{P}_{i,i}x_{i,{\rm now}}+ \frac{2w_{\ell-1}}{|\la_2|w_{\ell}}\sum_{j\in \Nc(i)}\mathbf{P}_{i,j}x_{j,{\rm now}}-\frac{w_{\ell-2}}{w_{\ell}}x_{i,{\rm prev}}$

\end{algorithm}

Chebyshev polynomials \cite{young2014iterative} are defined as $T_0(x)=1, T_1(x)=x$ and $T_{k+1}(x)=2xT_k(x)-T_{k-1}(x)$. Define:
\begin{align}
    q_\ell(\mathbf{P}) = \frac{T_{\ell}(\mathbf{P}/|\la_2|)}{T_\ell(1/|\la_2|)}.
\end{align}

By the properties of Chebyshev polynomial \cite{arioli2014chebyshev}, it can be shown that:
\begin{align}
    q_{\ell+1}(\mathbf{P})= \frac{2w_\ell}{|\la_2|w_{\ell+1}}\mathbf{P}q_{\ell}(\mathbf{P})-\frac{w_{\ell-1}}{w_{\ell+1}}q_{\ell-1}(\mathbf{P}),
\end{align}
where $w_0 = 1, w_1 = 1/|\la_2|$, $w_{\ell+1}=2w_\ell/|\la_2|-w_{\ell-1}$, $q_0(\mathbf{P})=I$ and $q_1(\mathbf{P})=\mathbf{P}$. This implies that when agents share an specific quantity, whose initial values given by agents are denoted by vector $\boldsymbol \nu_0\in \mathbb{R}^N$, by using the recursive Chebyshev-accelerated updating rule, they have:
\begin{align}
  \boldsymbol \nu_{\ell+1} = \frac{2w_\ell}{|\la_2|w_{\ell+1}}\mathbf{P}\boldsymbol \nu_{\ell}-\frac{w_{\ell-1}}{w_{\ell+1}}\boldsymbol \nu_{\ell-1}.
\end{align}
In light of the above mentioned recursive procedure, the accelerated communication step is summarized in Algorithm \ref{alg:comm} below for agent $i$. We denote the inputs by: 1) $x_{\rm now}$, which is the quantity of interest that agent $i$ wants to update at the current round, 2) $x_{\rm prev}$, which is the estimated value for a quantity of interest that agent $i$ updated at the previous round, and 3) $\ell$ which is the current round of communication. Note that inputs are scalars, however matrices and vectors also can be passed as inputs with Comm running for each of their entries.

%%%%%%%%%%%%%%%%%%%%%%%%%%%%%%%%%%%%%%%%%%%%%%%%%%%%%%%%%%%%%%%%%%%%%%%%%%%%%%%%%%%%%%%%%%%%%%%%%%%%%%%%%%%%%%%%%%%%%%%%%%%%%%%%%%%%%%%%%%%%%%%%%%%%%%%%%%%%%%%%%%%%%%%%%%%%%%%%%%%%%%%%%%%%%%%%%%%%%%%%%%%%%%%%%%%%%%%%%%%%%%%%%%%%%%%%%%%%%%%%%%%%%%%%%%%%%%%
\section{OMITTED ALGORITHMS}\label{sec:omittedalgorithms}
In this section, we present a definition and necessary algorithms, that are borrowed from \cite{ruan2021linear} and are used as subroutines in \DisBE and \DecBE.
\begin{definition}[\cite{ruan2021linear}]\label{def:softmaxandmixedsoftmax}
Fix $\alpha = \log K$. For a given positive semi-definite matrix $\M$, we define the softmax policy $\pi_{\M}^{\rm S}(\Xc)$ over a set $\Xc = \{\x_1,\x_2,\ldots,\x_k\}$ with $k\leq K$ with
\begin{align}
    \pi_{\M}^{\rm S}(\x_i) = \frac{(\x_i^\top\M\x_i)^\alpha}{\sum_{i=1}^k(\x_i^\top\M\x_i)^\alpha}.
\end{align}
Now, suppose we are given a set $\Mc = \left\{(p_i,\M_i)\right\}_{i=1}^n$ such that $p_i\geq 0$ and $\sum_{i=1}^np_i=1$. We
define the mixed-softmax policy $\pi^{{\rm MS}}_{\Mc}(\Xc)$ over $\Xc$ as
\begin{align}
    \pi^{{\rm MS}}_{\Mc}(\x_i) = \begin{cases}
      \piG(\Xc), &\hspace{0.1in} \text{with probability $1/2$}, \\
      \pi_{\M_i}^{\rm S}(\Xc), &\hspace{0.1in} \text{with probability $p_i/2$},
    \end{cases}
\end{align}
where $\piG(\Xc)$ is called $G$-optimal design and is the maximizer of $g(\pi)=\max_{\x\in\Xc}\norm{\x}_{\Vb(\pi)^{-1}}^2$, where $\Vb(\pi)= \sum_{\x\in\Xc}\pi(\x)\x\x^\top$; see Section 21 in \cite{lattimore2020bandit} for details.
\end{definition}

\begin{algorithm}[ht]
   \caption{\Exp}
   \label{alg:exppolicy}
   \DontPrintSemicolon
\KwInput{ $\la$, $\Sc = \{\Xc_1,\Xc_2,\ldots,\Xc_L\}$}
 {\bf Output:} A mixed-softmax policy $\pi$\;
 Using Algorithm \ref{alg:findcore} find a core $\Cc\subseteq\Sc$  such that 
 \begin{align}
 \max_{\Xc_i\in\Cc,\x\in\Xc_i} \x^\top\A(\Cc)^{-1}\x> d^5
 \end{align}
 and 
 \begin{align}
     \frac{\abs{\Cc}}{L}< 1-\Oc(d^{-2}\log\la^{-1})
 \end{align}
where $\A(\Cc) := \la I+\frac{1}{L}\sum_{\Xc_i\in\Cc}\mathbb{E}_{\x\sim\piG(\Xc_i)}[\x\x^\top]$, and for any set $\Xc\subset \mathbb{R}^d$, $\piG(\Xc)$ is called $G$-optimal design and is the maximizer of $g(\pi)=\max_{\x\in\Xc}\norm{\x}_{\Vb(\pi)^{-1}}^2$, where $\Vb(\pi)= \sum_{\x\in\Xc}\pi(\x)\x\x^\top$.\;
Return the mixed-softmax policy $\pi$ by calling $\MixedSoftMax(\la,\Cc)$.
\end{algorithm}

\begin{algorithm}[ht]
   \caption{CoreIdentification (Algorithm 4 in \citep{ruan2021linear})}
   \label{alg:findcore}
   \DontPrintSemicolon
\KwInput{ $\la$, $\Sc = \{\Xc_1,\Xc_2,\ldots,\Xc_L\}$}
 {\bf Output:} A core set $\Cc\subseteq\Sc$\;
   {\bf Initialization:}  $\Cc_1=\Sc$\;
  
  \For{$\xi=1,2,\ldots$}{
  \If{$\max_{\Xc_i\in\Cc_\xi,\x\in\Xc_i} \x^\top\A(\Cc_\xi)^{-1}\x> d^5$}
  {Return $\Cc_\xi$.}
  \Else
  {\begin{align}
      \Cc_{\xi+1}=\left\{\Xc_i\in \Cc_\xi: \max_{\x\in\Xc_i} \x^\top\A(\Cc_\xi)^{-1}\x\leq \frac{1}{2}d^5\right\},\nn
  \end{align}}
  where $\A(\Cc) := \la I+\frac{1}{L}\sum_{\Xc_i\in\Cc}\mathbb{E}_{\x\sim\piG(\Xc_i)}[\x\x^\top]$, and for any set $\Xc\subset \mathbb{R}^d$, $\piG(\Xc)$ is called $G$-optimal design and is the maximizer of $g(\pi)=\max_{\x\in\Xc}\norm{\x}_{\Vb(\pi)^{-1}}^2$, where $\Vb(\pi)= \sum_{\x\in\Xc}\pi(\x)\x\x^\top$.
  }
\end{algorithm}

\begin{algorithm}[ht]
% \numbering
  \caption{\MixedSoftMax}
  \label{alg:MixedSoftMax}
   \DontPrintSemicolon
\KwInput{$\la$, $\Sc = \{\Xc_1,\Xc_2,\ldots,\Xc_L\}$}
{\bf Output:} A mixed-softmax policy $\pi$\;
{\bf Initialization:} $Q= 2d^2\log d$, $\Xc_{(i-1)L+j}=\Xc_j,~\forall (i,j)\in[Q]\times L$, $\Ub_0= \la QLI+\frac{Q}{2}\sum_{i=1}^L\mathbb{E}_{\x\sim\piG(\Xc_i)}[\x\x^\top]$, $n=1$, $\tau_n=\emptyset$, $\W_n=\Ub_0$\;
\For{$s=1,\ldots,QL$}{
$\tau_n = \tau_n\cup \{s\}$\;
$\Ub_s = \Ub_{s-1}+\mathbb{E}_{\x\sim\pi^{\rm S}_{\W_n^{-1}}(\Xc_s)}[\x\x^\top]$, where $\pi^{\rm S}_{\W_n^{-1}}(\Xc_s)$ is computed as in Definition \ref{def:softmaxandmixedsoftmax}.\;
\If{$\frac{\det\Ub_s}{\det \W_n}>2$}{
$n = n+1$, $\tau_n=\emptyset$, $\W_n=\Ub_s$
}
}
$p_i = \frac{\mathbbm{I}\{\abs{\tau_i}\geq L\}\abs{\tau_i}}{\sum_{i=1}^n\mathbbm{I}\{\abs{\tau_i}\geq L\}\abs{\tau_i}}$ and $\M_i = QL\W^{-1}_i,~\forall i\in[n]$\;
Return the mixed-softmax policy with parameters $\Mc = \left\{(p_i,\M_i)\right\}_{i=1}^n$ as in Definition \ref{def:softmaxandmixedsoftmax}.
\end{algorithm}

%%%%%%%%%%%%%%%%%%%%%%%%%%%%%%%%%%%%%%%%%%%%%%%%%%%%%%%%%%%%%%%%%%%%%%%%%%%%%%%%%%%%%%%%%%%%%%%%%%%%%%%%%%%%%%%%%%%%%%%%%%%%%%%%%%%%%%%%%%%%%%%%%%%%%%%%%%%%%%%%%%%%%%%%%%%%%%%%%%%%%%%%%%%%%%%%%%%%%%%%%%%%%%%%%%%%%%%%%%%%%%%%%%%%%%%%%%%%%%%%%%%%%%%%%%%%%%%
\section{AUXILIARY LEMMAS}\label{sec:auxiliary}
\begin{lemma}[\cite{tropp2015introduction}, Theorem 5.1.1]\label{lemm:matrixchernoff}
Consider a finite sequence $\X_k$ of independent, random, Hermitian matrices with common dimension $d$. Assume that $0\leq \lamin(\X_k)$ and $\lamax(\X_k)\leq L$ for each index $k$. Introduce the random matrix 
\begin{align}
    \Yb= \sum_{k=1}^n\X_k
\end{align}
Define the minimum eigenvalue $\mu_{\rm min}$ and maximum eigenvalue $\mu_{\rm max}$ of the expectation $\mathbb{E}[\Yb]$:
\begin{align}
    \mu_{\rm min} = \lamin(\mathbb{E}[\Yb]),\quad\mu_{\rm max} = \lamax(\mathbb{E}[\Yb]).
\end{align}
Then
\begin{align}
    \mathbb{P}\left(\lamin(\Yb)\leq (1-\varepsilon)\mu_{\rm min}\right)&\leq d\left(\frac{\exp(-\varepsilon)}{(1-\varepsilon)^{1-\varepsilon}}\right)^{\frac{\mu_{\rm min}}{L}},\quad \text{for}~ \varepsilon\in[0,1)\\
    \mathbb{P}\left(\lamax(\Yb)\geq (1+\varepsilon)\mu_{\rm max}\right)&\leq d\left(\frac{\exp(\varepsilon)}{(1+\varepsilon)^{1+\varepsilon}}\right)^{\frac{\mu_{\rm max}}{L}},\quad \text{for}~ \varepsilon\geq 0.\label{eq:chernofflamax}
\end{align}
\end{lemma}
\begin{lemma}\label{lemm:upperboundonrandommatrix}
Suppose $\x_1,\x_2,\ldots,\x_n\sim\Dc$ are $d$-dimensional vectors that are i.i.d. drawn from a distribution $\Dc$ and $\norm{\x_k}_2\leq L$ for all $k\in[n]$ almost
surely. Let $\gamma = \lamin\left(\mathbb{E}_{\x\sim\Dc}[\x\x^\top] \right)> 0$ be the smallest eigenvalue of the co-variance matrix. We
have that
\begin{align}
    \mathbb{P}\left(\frac{1}{n}\sum_{k=1}^n\x_k\x_k^\top\preceq 2\mathbb{E}_{\x\sim\Dc}[\x\x^\top]\right)\geq 1-d\exp\left(\frac{-\gamma n}{3}\right).
\end{align}
\end{lemma}
\begin{proof}
Let $\Sigmab=\mathbb{E}_{\x\sim\Dc}[\x\x^\top]$ and $\y_k = \Sigmab^{\frac{-1}{2}}\x_k$ for all $k\in[n]$. Also, we have $\lamax(\y_k\y_k^\top) = \norm{\y_k}_2^2\leq \frac{1}{\gamma}$ almost
surely, and $\mathbb{E}[\y_k\y_k^\top]=I$. Therefore, plugging $\varepsilon=1$ in \eqref{eq:chernofflamax}, we have
\begin{align}
\mathbb{P}\left(\frac{1}{n}\sum_{k=1}^n\x_k\x_k^\top\preceq 2\mathbb{E}_{\x\sim\Dc}[\x\x^\top]\right)&=\mathbb{P}\left(\frac{1}{n}\sum_{k=1}^n\y_k\y_k^\top\preceq 2\Sigmab^{\frac{-1}{2}}\mathbb{E}_{\x\sim\Dc}[\x\x^\top]\Sigmab^{\frac{-1}{2}}\right)\nn\\
&=\mathbb{P}\left(\frac{1}{n}\sum_{k=1}^n\y_k\y_k^\top\preceq 2I\right)\nn\\
&=\mathbb{P}\left(\lamax\left(\sum_{k=1}^n\y_k\y_k^\top\right)\leq2n\right)\nn\\
&\geq 1 - d\left(\frac{e}{4}\right)^{n\gamma}\geq 1-d\exp\left(\frac{-\gamma n}{3}\right).
\end{align}
\end{proof}

\begin{lemma}\label{lemm:upperboundonrandommatrixwithcuttoff}
Suppose $\x_1,\x_2,\ldots,\x_n\sim\Dc$ are $d$-dimensional vectors that are i.i.d. drawn from a distribution $\Dc$ and $\norm{\x_k}_2\leq 1$ for all $k\in[n]$ almost
surely. For any cutoff level $\gamma>0$, we have
\begin{align}
    \mathbb{P}\left(\frac{1}{n}\sum_{k=1}^n\x_k\x_k^\top\preceq 2\mathbb{E}_{\x\sim\Dc}[\x\x^\top]+6\gamma I\right)\geq 1-2d\exp\left(\frac{-\gamma n}{3}\right).
\end{align}
\end{lemma}
\begin{proof}
Suppose $\mathbb{E}_{\x\sim\Dc}[\x\x^\top] = \sum_{i=1}^d\la_i\nub_i\nub_i^\top$, where $\{\nub_i\}_{i=1}^d$ is a set of orthonormal basis. Let $\Pb_+=\sum_{i=1}^d\nub_i\nub_i^\top\mathbbm{1}(\la_i\geq \gamma)$ and $\Pb_-=\sum_{i=1}^d\nub_i\nub_i^\top\mathbbm{1}(\la_i<\gamma)$, so that $\Pb_+\Pb_- = I$. We observe that the
eigenvalues of $\mathbb{E}_{\x\sim\Dc}[\Pb_+\x\x^\top\Pb_+^\top]$ are greater than or equal to $\gamma$ when restricted to the space spanned by the $\Pb_+$. Therefore, by Lemmas \ref{lemm:upperboundonrandommatrix} and \ref{lemm:matrixchernoff} (Eqn. \eqref{eq:chernofflamax}), we respectively have
\begin{align}
    \mathbb{P}\left(\frac{1}{n}\sum_{k=1}^n\Pb_+\x_k\x_k^\top\Pb_+^\top\preceq 2\mathbb{E}_{\x\sim\Dc}[\Pb_+\x\x^\top\Pb_+^\top]\right)&\geq 1-d\exp\left(\frac{-\gamma n}{3}\right)\\
     \mathbb{P}\left(\frac{1}{n}\sum_{k=1}^n\Pb_-\x_k\x_k^\top\Pb_-^\top\preceq 2\gamma I\right)&\geq 1-d\exp\left(\frac{-\gamma n}{3}\right).\label{eq:begining}
\end{align}
Now, we observe that
\begin{align}
    \frac{1}{n}\sum_{k=1}^n\x_k\x_k^\top &= \frac{1}{n}\left(\sum_{k=1}^n\Pb_+\x_k\x_k^\top\Pb_+^\top+\sum_{k=1}^n\Pb_+\x_k\x_k^\top\Pb_-^\top+\sum_{k=1}^n\Pb_-\x_k\x_k^\top\Pb_+^\top+\sum_{k=1}^n\Pb_-\x_k\x_k^\top\Pb_-^\top\right)\nn\\
    &=\frac{1}{n}\left(\sum_{k=1}^n\Pb_+\x_k\x_k^\top\Pb_+^\top+\sum_{k=1}^n\Pb_+\Pb_+\Pb_-\x_k\x_k^\top\Pb_-^\top+\sum_{k=1}^n\Pb_-\x_k\x_k^\top\Pb_-^\top\Pb_+^\top\Pb_+^\top+\sum_{k=1}^n\Pb_-\x_k\x_k^\top\Pb_-^\top\right)\nn\\
    &\preceq \frac{1}{n}\left(\sum_{k=1}^n\Pb_+\x_k\x_k^\top\Pb_+^\top+\sum_{k=1}^n\Pb_-\x_k\x_k^\top\Pb_-^\top+\sum_{k=1}^n\Pb_-\x_k\x_k^\top\Pb_-^\top+\sum_{k=1}^n\Pb_-\x_k\x_k^\top\Pb_-^\top\right)\nn\\
    &= \frac{1}{n}\sum_{k=1}^n\Pb_+\x_k\x_k^\top\Pb_+^\top+\frac{3}{n}\sum_{k=1}^n\Pb_-\x_k\x_k^\top\Pb_-^\top\label{eq:middle}
\end{align}
Also, note that
\begin{align}
    \mathbb{E}_{\x\sim\Dc}[\Pb_+\x\x^\top\Pb_+^\top] &= \mathbb{E}_{\x\sim\Dc}\left[\x\x^\top-\Pb_+\x\x^\top\Pb_-^\top-\Pb_-\x\x^\top\Pb_+^\top-\Pb_-\x\x^\top\Pb_-^\top\right]\preceq\mathbb{E}_{\x\sim\Dc}\left[\x\x^\top\right].\label{eq:last}
\end{align}
Therefore, combining \eqref{eq:begining} and \eqref{eq:middle} and \eqref{eq:last}, we have
\begin{align}
    \mathbb{P}\left(\frac{1}{n}\sum_{k=1}^n\x_k\x_k^\top\preceq 2\mathbb{E}_{\x\sim\Dc}[\x\x^\top]+6\gamma I\right)\geq 1-2d\exp\left(\frac{-\gamma n}{3}\right).
\end{align}

\end{proof}

\end{document}